%% file: main.tex
\documentclass{article} 
\usepackage{iclr2025_conference,times}

\input{math_commands.tex}

\usepackage{hyperref}
\usepackage{url}

\usepackage{subcaption}
\usepackage{amsmath}
\usepackage{amssymb}
\usepackage{cleveref}
\usepackage{booktabs}
\usepackage{multirow}
\newcommand{\ie}{\emph{i.e.}}

\usepackage{algorithm}
\usepackage{algorithmic}

\usepackage{array}
\usepackage{sidecap}

\usepackage{textgreek}
\usepackage{tikz}
\usepackage{mathtools}
\usepackage{amsthm}
\newtheorem{definition}{Definition}
\newtheorem{theorem}{Theorem}
\newtheorem{lemma}{Lemma}
\newtheorem{corollary}{Corollary}

\usepackage{wrapfig}

\title{Long-tailed Adversarial Training with Self-Distillation}

\author{
Seungju Cho\thanks{These authors contributed equally.}\textsuperscript{$*$}, 
Hongsin Lee\footnotemark[1]\textsuperscript{$*$}, 
Changick Kim \\
Korea Advanced Institute of Science and Technology (KAIST)\\
{\normalsize \texttt{\{joyga, hongsin04, changick\}@kaist.ac.kr}} \\ 
}

\iclrfinalcopy 
\begin{document}

\maketitle

\begin{abstract}
 Adversarial training significantly enhances adversarial robustness, yet superior performance is predominantly achieved on balanced datasets.
 Addressing adversarial robustness in the context of unbalanced or long-tailed distributions is considerably more challenging, mainly due to the scarcity of tail data instances. 
 Previous research on adversarial robustness within long-tailed distributions has primarily focused on combining traditional long-tailed natural training with existing adversarial robustness methods.
 In this study, we provide an in-depth analysis for the challenge that adversarial training struggles to achieve high performance on tail classes in long-tailed distributions.
 Furthermore, we propose a simple yet effective solution to advance adversarial robustness on long-tailed distributions through a novel self-distillation technique.
 Specifically, this approach leverages a balanced self-teacher model, which is trained using a balanced dataset sampled from the original long-tailed dataset.
Our extensive experiments demonstrate state-of-the-art performance in both clean and robust accuracy for long-tailed adversarial robustness, with significant improvements in tail class performance on various datasets.
We improve the accuracy against PGD attacks for tail classes by 20.3, 7.1, and 3.8 percentage points on CIFAR-10, CIFAR-100, and Tiny-ImageNet, respectively, while achieving the highest robust accuracy.

\end{abstract}

\section{Introduction}
Recent studies have highlighted the vulnerabilities inherent in deep learning models when subjected to adversarial attacks \citep{FGSM, CW_attack, PGD, athalye2018obfuscated}.
These attacks exploit subtle changes in input data that can lead to drastically incorrect predictions, undermining model reliability in critical applications \citep{safe1, safe2, wang2023does}.
As a result, research efforts have focused on enhancing robustness against such adversarial threats, with various strategies being explored \citep{2017_jpeg_defense, xie2019feature, cohen2019certified, carmon2019unlabeled, zhang2022adversarial, jin2023randomized}.
Among these, adversarial training \citep{FGSM, PGD} has proven to be one of the most effective methods for enhancing model robustness \citep{pang2020bag, bai2021recent, wei2023cfa}.
However, many existing studies primarily validate their approaches on balanced datasets, overlooking the practical scenarios where data is inherently imbalanced or long-tailed.
This gap underscores the need for novel adversarial training methodologies capable of addressing these more complex data distributions effectively.

\begin{figure}[t]
\centering
    \includegraphics[width=1\textwidth] {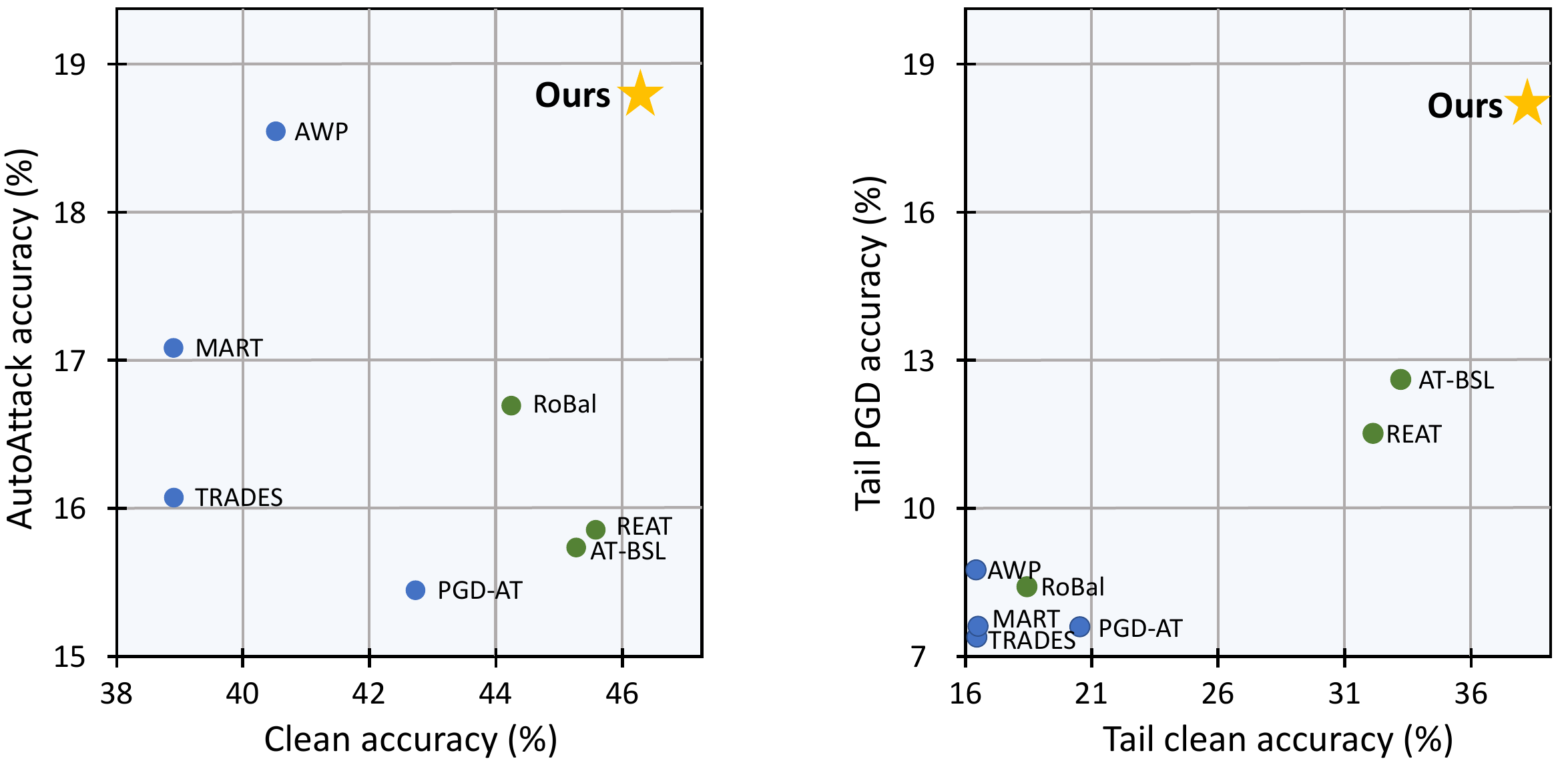}\\
    \vspace{-0.5cm}
    \begin{minipage}[t]{.485\textwidth}
      \subcaption{Performance on all classes}   
               \label{fig:intro1-a}
    \end{minipage}%
    \begin{minipage}[t]{.485\textwidth}
        \subcaption{Performance on tail classes}
         \label{fig:intro1-b}
    \end{minipage}%
    \caption{\textbf{(a)} The overall clean accuracy and AutoAttack \citep{AutoAttack} accuracy of various adversarial training methods (\textit{blue circles}) and long-tailed adversarial training methods (\textit{green circles}) using the ResNet-18 \citep{he2016deep} architecture on CIFAR-100-LT \citep{krizhevsky2009learning}. 
    \textbf{(b)} The clean accuracy and 20-step PGD attack \citep{PGD} accuracy on tail classes for the same set of methods.
    Our method (\textit{yellow star}) surpasses all existing methods, achieving a notable improvement on tail classes.
    }
    \label{fig:intro1}
\end{figure}

While numerous studies \citep{cao2019learning, cui2019class, kang2019decoupling, zhou2020bbn, li2021self, alshammari2022long, du2023global} have addressed long-tailed distributions without considering robustness, the intersection of adversarial training and long-tailed distributions \citep{Wu_2021_CVPR, li2023alleviating, Yue_2024_CVPR} has received far less attention.
Existing research in this area primarily combines traditional long-tailed classification techniques with basic adversarial training methods, such as PGD adversarial training \citep{PGD} and TRADES loss \citep{TRADES} with balanced softmax \citep{Wu_2021_CVPR, Yue_2024_CVPR}.
Despite combining such methods, existing approaches still demonstrate low performance on tail classes with fewer samples in long-tailed distributions.
We find that their high robustness primarily stems from the improved robustness of head classes, which have a larger number of samples.
This highlights the need for more advanced research on adversarial training in long-tailed distributions.

In this paper, we provide an in-depth analysis of why adversarial training in long-tailed distributions is particularly challenging, focusing on the performance on tail classes.
Through theoretical analysis, we show that adversarial training causes more severe performance degradation in tail classes compared to natural training.
This highlights the inherent difficulty of achieving high robustness in long-tail distributions, especially for the tail classes.
Building on these insights, we propose a novel two-step framework designed to improve tail class robustness under adversarial training on long-tailed distribution. 

Our framework consists of constructing a balanced dataset from a given unbalanced dataset and employing self-distillation.
We first create a sub-dataset where each class contains an equal number of data samples, referred to as the balanced sub-dataset.
Then, we adversarially train a self-teacher model on this balanced dataset, achieving higher robustness in tail classes than models trained on the full long-tailed dataset.
Subsequently, we apply self-distillation using the balanced self-teacher model to improve tail class performance, resulting in significant gains over baseline models.
As shown in \Cref{fig:intro1-a}, 
our method achieves the highest accuracy against AutoAttack \citep{AutoAttack} and demonstrates significant performance improvements, particularly on tail classes as in \Cref{fig:intro1-b}.
Here are our key contributions:
\begin{itemize}
\item 
We conduct an in-depth analysis to explain why adversarial training on long-tailed datasets results in poor tail class performance. Our findings show that, despite adversarial training, tail class robustness is even lower than natural training.
\item Based on these insights, we propose a novel two-step adversarial training approach specifically designed for long-tailed datasets.
This method improves upon baselines that merely combine existing long-tailed classification techniques with adversarial training.
\item Our approach achieves state-of-the-art performance in adversarial training on long-tailed datasets across various architectures, datasets, and imbalance ratios, leading to significant enhancements in both clean and robust accuracy, with particularly notable improvements on tail classes.
\end{itemize}

\section{Related Works}
\subsection{Adversarial Training and Distillation}

In response to adversarial attacks \citep{FGSM, CW_attack, PGD, athalye2018obfuscated}, 
adversarial training \citep{FGSM, PGD} empirically stands out as one of the most effective.
Adversarial training defines optimization as a min-max problem, where inner maximization generates adversarial inputs, and outer minimization trains the model on these adversarial samples. 
TRADES \citep{TRADES} incorporates the Kullback-Leibler (KL) divergence loss between the logits of clean and adversarial images.
MART \citep{MART} introduces per-sample weights based on the confidence of each sample.
These two methods are used as baseline methods for other recent adversarial training research \citep{qin2019adversarial, 2020_awp,bai2021improving, jin2022enhancing, tack2022consistency, jin2023randomized,wei2023cfa}. 

The superior performance of adversarial training is primarily observed in large architecture networks, motivating research efforts to improve performance in smaller architectures using techniques such as distillation.
Adversarial Robust Distillation (ARD) \citep{ard} proposes a loss function that guides the adversarial output of the student model towards the natural output of the teacher, similar to TRADES \citep{TRADES}.
Robust Soft Labels Adversarial Distillation (RSLAD) \citep{rslad} leverages teacher logits to improve performance through inner maximization in adversarial training.
Many other studies leverage the teacher's logits \citep{iad, akd, adaad} and gradients \citep{IGDM} to distill robustness into the student model.

While these adversarial training and distillation studies have achieved strong robustness, they have only been conducted on balanced datasets where each class has an equal number of samples.
This differs significantly from the real-world data configurations we encounter, highlighting the necessity of adversarial training or distillation for unbalanced datasets.

\subsection{Long-Tailed Recognition}

Extensive research has been conducted to address the performance imbalance inherent in long-tailed distribution datasets.
Prominent methods include oversampling the minority tail data \citep{chawla2002smote, han2005borderline} and increasing the weight of the minority classes \citep{cui2019class, zhang2021distribution}. Although these methods are intuitive, they pose a risk of overfitting on the tail classes and can degrade feature extraction performance \citep{kang2019decoupling, zhou2020bbn}.
A more effective approach, decoupled learning \citep{kang2019decoupling, zhou2020bbn, alshammari2022long}, separates feature learning from classification to mitigate such issues.
Moreover, logit compensation methods have been proposed recently, introducing relatively larger margins between different classes based on prior class frequencies \citep{cao2019learning, kang2019decoupling, menon2020long, ren2020balanced, tan2020equalization}.

\subsection{Long-Tailed Adversarial Training}

RoBal \citep{Wu_2021_CVPR} is the first paper to address adversarial robustness in long-tailed classification. RoBal applies adversarial training with TRADES regularization \citep{TRADES} alongside long-tailed techniques such as balanced softmax \citep{ren2020balanced} and class-aware margin \citep{class_margin_lt}.
Moreover, it provides detailed insights into which modules are effective for long-tailed adversarial training.
REAT \citep{li2023alleviating} aimed to achieve balanced performance by utilizing class-wise weights to generate adversarial examples and expanding the feature space of tail class data.
AT-BSL \citep{Yue_2024_CVPR} revisited the RoBal paper to analyze the necessity of various modules and concluded that only the balanced softmax loss (BSL) is sufficient without the need for complex modules as follows:
\begin{equation}
\label{eq:bsl}
    \mathcal{L}_{BSL}(f(\boldsymbol{x}'),y) = -\log\Big(\frac{e^{z'_{y} + b_y}}{\sum_{i}{e^{z'_{i}+ b_i}}}\Big),
\end{equation}
where $\boldsymbol{x}'$ is an adversarially perturbed input of $\boldsymbol{x}$, $z'_i = f(\boldsymbol{x}')_i$ is $i$-th logits of the adversarial input, $b_i = \tau \ log(n_i)$, $\tau$ is a hyperparameter and $n_i$ is the number of examples in the $i$-th class.
The balanced softmax is a commonly used loss in addressing long-tail problems to boost the performance of tail classes \citep{ren2020balanced}.
However, its drawback lies in adjusting the importance of tail classes based on the number of data.
In other words, more than balanced softmax is needed to address robustness concerns for tail classes adequately.

There has been no in-depth analysis of robustness degradation in tail classes compared to head classes in adversarially robust long-tailed distributions.
In this paper, we aim to improve the overall performance of existing long-tailed adversarial training by achieving sufficient robustness of tail classes.

\section{Analysis}

\subsection{Preliminary}
Let $f$ represent the classification model, which maps the input data space $\mathcal{X}$ to the output label space $\mathcal{Y}$, i.e., $f: \mathcal{X} \rightarrow \mathcal{Y}$.
For specific instance of $\mathcal{X}$ and $\mathcal{Y}$, we use $\boldsymbol{x} \in \mathcal{X}$ and $y \in \mathcal{Y}$ and $\boldsymbol{x} = (x_1,x_2, \dots, x_n)$ where $n$ is the dimension of $\boldsymbol{x}$. 
\begin{definition}
 For a classifier $f(\cdot)$, the overall standard error $\mathcal{R}_\text{std}(\cdot)$ of classifier $f(\cdot)$  is defined as
\[
\mathcal{R}_\text{std}(f) = \Pr(f(\boldsymbol{x}) \neq y),
\]
and its robust error $\mathcal{R}_\text{rob}(\cdot)$ is
\[
\mathcal{R}_\text{rob}(f) = \Pr(\exists \boldsymbol{\delta} \ \text{ with } \ \|\boldsymbol{\delta}\|_\infty \leq \epsilon \text{ s.t. } f(\boldsymbol{x} + \boldsymbol{\delta}) \neq y)
\]
where $\Pr(\cdot)$ means probability and $\epsilon$ is a non-negative perturbation boundary.
\end{definition}
For simplicity, we denote by $f_{nat}(\cdot)$ the natural classifier that minimizes standard error, and by $f_{rob}(\cdot)$ the robust classifier that minimizes robust error. Additionally, we denote the standard error and robust error for a given class $k$ as $\mathcal{R}^k_\text{std}(f)$ and $\mathcal{R}^k_\text{rob}(f)$, respectively.

\begin{figure}[b]
  \begin{center}
    \includegraphics[width=1\textwidth]{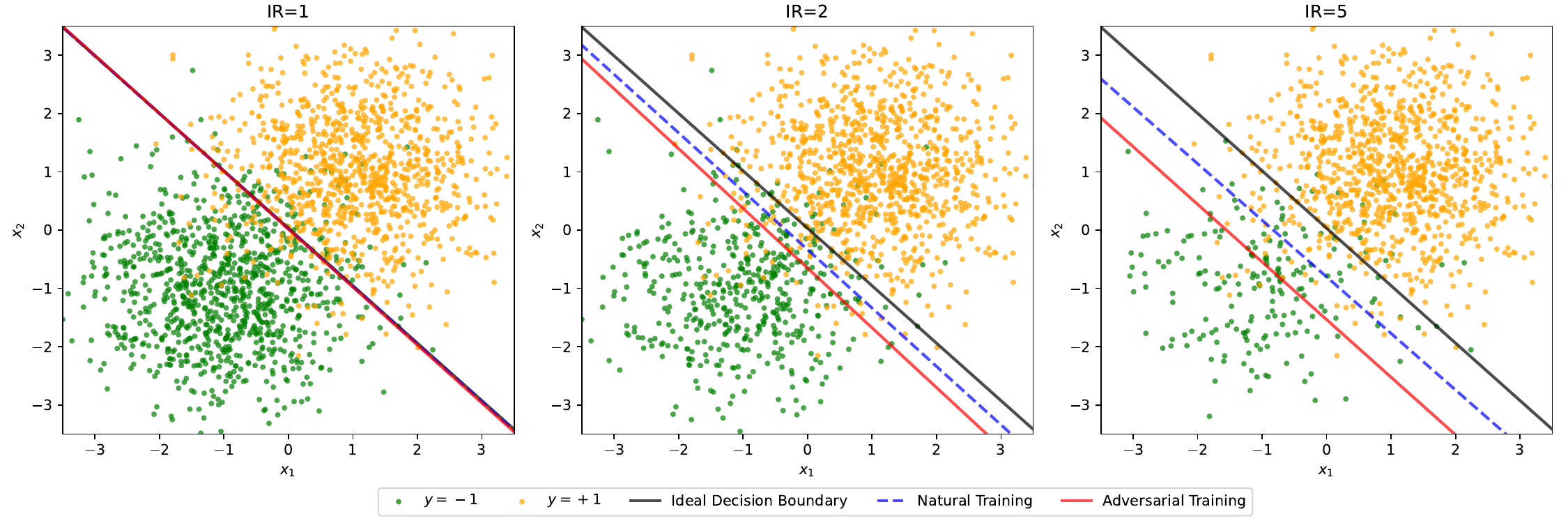}
  \end{center}
  \caption{Logistic regression on binary data in \Cref{eq:binary_data} with different imbalance ratio (IR).}
  \label{fig:logistic_fig}
  \vspace{-0.5cm}
\end{figure}

\subsection{Theoretical Analysis}

Let the long-tailed dataset for a binary classification task data $\mathcal{S}$ with imbalance ratio $r\geq1$, \ie, the ratio of the number of instances in the head class ($y = +1$) to the number of instances in the tail class ($y = -1$) is $r$.
We assume Gaussian mixture distribution, which is similar to \cite{FairAT, DAFA} as follows.
\begin{equation}
\label{eq:binary_data}
    y  = \left\{ \begin{array}{cl}
+1, & \text{w.p} \ \frac{r}{r+1} \\
-1, & \text{w.p} \ \frac{1}{r+1}
\end{array} \right. , \ \ 
x_1, \cdots, x_{n} \overset{i.i.d}{\sim} \mathcal{N}(\eta y, 1),
\end{equation}
where $\eta > \epsilon $ is a constant that determines the standard deviation of the Gaussian distribution.
We address a binary classification problem on the above dataset, and then we obtain the following linear function $ f_{\boldsymbol{w}, b}(\cdot)$, with weight $\boldsymbol{w}$ with bias $b$.
\begin{equation}
    f_{\boldsymbol{w}, b}(\boldsymbol{x}) = \text{sign}\left(\sum_{k=1}^{n} w_k x_k + b\right).
    \end{equation}
According to \Cref{lemma:identical_weight_natural} and \Cref{lemma:identical_weight_robust}  in \Cref{sec:proof},  each weight $w_1,w_2, \cdots , w_n$ of optimal (natural, robust)  classifier has the same weight, \ie, $w_1 = w_2 = \cdots = w_n$.
We derive the standard and robust error for the tail class of each optimal classifier as follows.
\begin{theorem} 
\label{theorem:1}
    For a data distribution $\mathcal{S}$, the optimal natural classifier $f^*_{nat}$ and robust classifier $f^*_{rob}$ exhibit the following standard and robust errors for the tail class $-1$ with perturbation margin $0 < \epsilon < \eta$, respectively:
    \begin{align}
        \mathcal{R}^{-1}_{nat}(f^*_{nat}) &= \Phi\left(-\sqrt{n}\eta + \frac{\ln r}{2\sqrt{n}\eta}\right), \ \ \ \ \ \ \ \ \ \  \mathcal{R}^{-1}_{rob}(f^*_{nat}) = \Phi\left(-\sqrt{n}(\eta-\epsilon) + \frac{\ln r}{2\sqrt{n}\eta}\right),  \label{eq:thm1_1} \\
        \mathcal{R}^{-1}_{nat}(f^*_{rob}) &= \Phi\left(-\sqrt{n}\eta + \frac{\ln r}{2\sqrt{n}(\eta-\epsilon)}\right),\label{eq:thm1_2} \ 
        \mathcal{R}^{-1}_{rob}(f^*_{rob}) = \Phi\left(-\sqrt{n}(\eta-\epsilon) + \frac{\ln r}{2\sqrt{n}(\eta-\epsilon)}\right).
    \end{align}
\end{theorem}
A detailed proof of \Cref{theorem:1} can be found in \Cref{sec:proof}.
From \Cref{theorem:1}, we can easily infer that both the natural and robust errors of the tail class for both the natural and robust classifiers increase monotonically with respect to the imbalance ratio $r$.
Building upon this, we present the following corollary:

\begin{corollary}
\label{corollary:1}
Adversarial training on long-tailed datasets exacerbates the vulnerability of the tail class, making them even less robust than under natural training : 
\begin{equation}
    \mathcal{R}^{-1}_{rob}(f^*_{rob}) > \mathcal{R}^{-1}_{rob}(f^*_{nat}).
\end{equation}
\end{corollary}
A proof of \Cref{corollary:1} is trivial according to \cref{eq:thm1_1} and \cref{eq:thm1_2} in \Cref{theorem:1}.


\subsection{Empirical Analysis}

\begin{wrapfigure}{r}{0.4\textwidth}
  \begin{center}
    \vspace{-0.7cm}
    \includegraphics[width=0.4\textwidth]{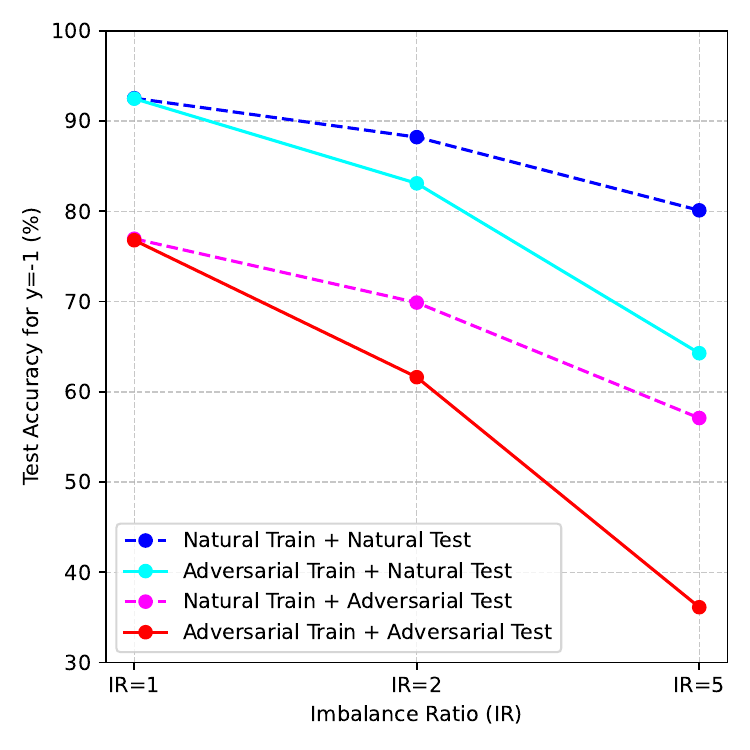}
  \end{center}
  \vspace{-0.5cm}
  \caption{Tail class natural and robust accuracy with respect to natural and adversarial training with different imbalance ratios (IR) in \Cref{fig:logistic_fig}.}
  \label{fig:logistic_acc}
  \vspace{-0.3cm}
\end{wrapfigure}

In \Cref{fig:logistic_fig}, we present a visualization of the theoretical analysis in a 2-dimensional space.
The data were sampled from Gaussian distributions with $\eta = 1$ and $n = 2$ following \Cref{eq:binary_data}, considering three different imbalance ratios (IR=1, 2, 5).
The figure highlights the decision boundaries formed by both natural and adversarial training ($\epsilon = 0.5$). 
As the imbalance ratio increases, the decision boundary of the adversarially trained model becomes more distorted, reflecting the model's increased sensitivity to adversarial perturbations in the minority class.
\begin{wrapfigure}{r}{0.5\textwidth}
    \vspace{-0.5cm}
    \centering
    \begin{minipage}[c]{.25\textwidth}
      \centering
      \includegraphics[scale=0.23]{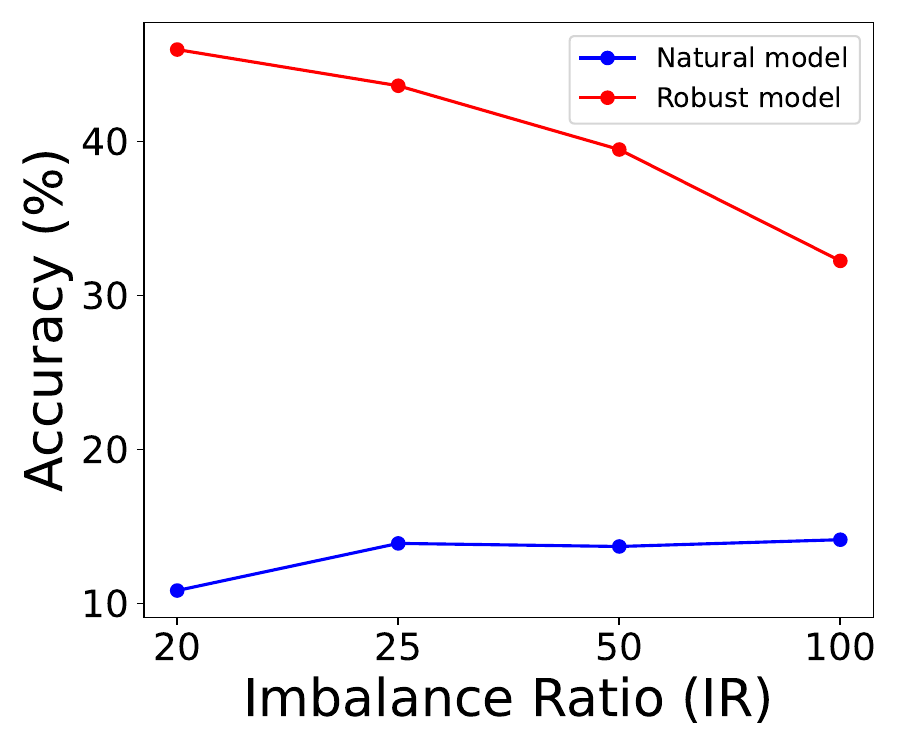}
    \end{minipage}%
    \begin{minipage}[c]{.25\textwidth}
      \centering
      \includegraphics[scale=0.23]{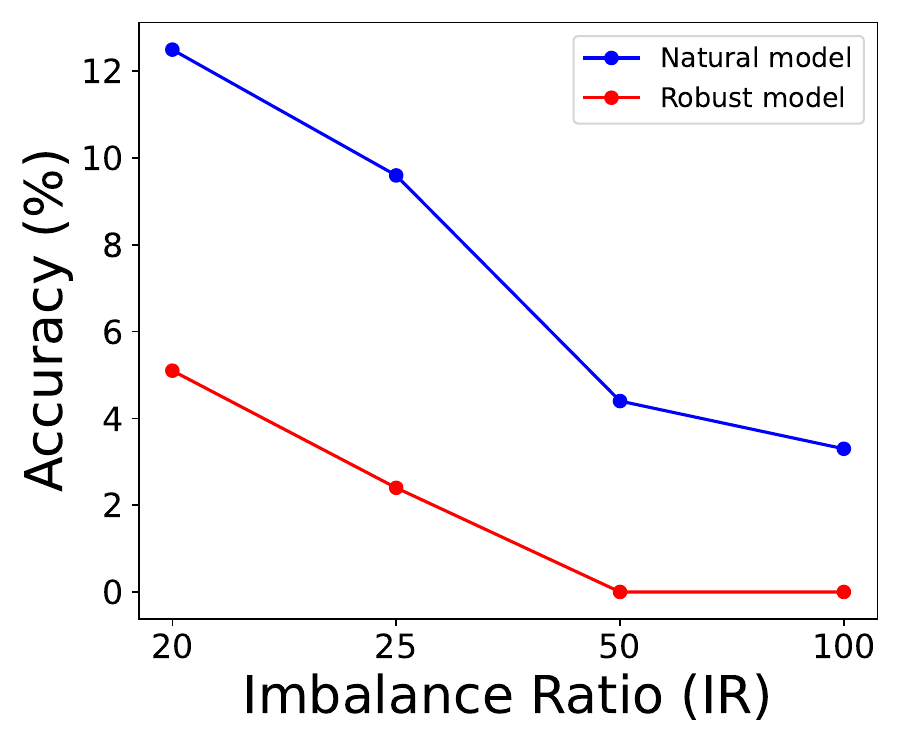}    
    \end{minipage}
      \\
    \begin{minipage}[t]{.25\textwidth}
      \subcaption{Whole classes.}   
    \label{fig:natural_vs_robust_a}
    \end{minipage}%
    \begin{minipage}[t]{.25\textwidth}
        \subcaption{Tail classes.}
        \label{fig:natural_vs_robust_b}
    \end{minipage}
    \caption{Clean and robust accuracy of natural and robust models.}%
    \label{fig:natural_vs_robust}%
    \vspace{-0.6cm}
\end{wrapfigure}



\Cref{fig:logistic_acc} shows the test accuracy of the tail class ($y=-1$) of the logistic regressions. The results indicate a clear trend: as the imbalance ratio grows, the test accuracy for the tail class drops across both natural and adversarial training scenarios. 
More notably, adversarial training consistently produces more robust errors than natural training.
These results align with our theoretical predictions: imbalanced data amplifies clean and robust errors for the tail class, and adversarial training further exacerbates robust errors.

We further experiment on a long-tailed CIFAR-10 dataset on various imbalance ratios.
We conduct both standard and PGD-adversarial training and compare the robust accuracy against PGD attack with $\epsilon = 2/255$ of the entire and tail classes.
As shown in \Cref{fig:natural_vs_robust_a}, the robust performance of the adversarially-trained model across various imbalance ratios is superior to that of the naturally-trained model, which shows trivial results.
However, in \Cref{fig:natural_vs_robust_b}, the natural model exhibits better robust performance than the adversarially-trained model in the tail classes.
This experiment supports the results of \Cref{corollary:1}, which clearly demonstrates that while adversarial training generally helps to improve robustness, it could exacerbate performance degradation in the tail classes of long-tailed distributions.

\begin{algorithm*}[t]
\caption{Main Algorithm}
\label{alg:algorithm}
\vspace{-0.02in}
\hspace*{0.1in}\textbf{Input:}  Long-tailed dataset $\mathcal{D}$, batch size $N$, epochs $T$, learning rate $\mu$, hyperparameters $\alpha$, $\gamma$,\\\hspace*{0.2in}and balanced self-teacher training parameters (batch size $N_B$, epochs $T_B$, and learning rate $\mu_B$)\\
\hspace*{0.1in}\textbf{Output:} Robust model $f$ on long-tailed dataset $\mathcal{D}$\vspace{0.1in}\\
\hspace*{0.1in}\textcolor{gray}{\# balanced self-teacher training} \\
\hspace*{0.1in}{Make balanced dataset $\mathcal{D}_{B}$ by re-sampling $\mathcal{D}$ with hyperparameter $\gamma$}\\
\hspace*{0.1in}Randomly initialize $\theta_B$, the weights of balanced self-teacher $f_{B}$ \\
\hspace*{0.1in}\textbf{for}{ epoch$\ = 1$ to $T_B$} \textbf{do} \\
\hspace*{0.2in}\textbf{for}{ mini batch $\{\boldsymbol{x}_i, y_i\}^{N_B}_{i=1}$ in $\mathcal{D}_{B}$} \textbf{do} \\ 
\hspace*{0.3in}\textbf{for}{ $i = 1$ to $N_B$} \textbf{do}\\
\hspace*{0.4in}\text{$\boldsymbol{x}'_{i}$ = PGD($\boldsymbol{x}_i$, $y_i$})\hspace*{2in}     \textcolor{gray}{\# PGD attack}\\ 
\hspace*{0.3in}\textbf{end for}\\
\hspace*{0.3in}$\theta_B \gets - \mu_B \frac{1}{N_B} \sum_{i=1}^{N_B}  \nabla_{\theta_B} \mathcal{L}_{CE}(f_B(\boldsymbol{x}'_i), y_i)$  \\
\hspace*{0.2in}\textbf{end for} \\
\hspace*{0.1in}\textbf{end for} \\

\hspace*{0.1in}\textcolor{gray}{\# main training} \\
\hspace*{0.1in}Randomly initialize $\theta$, the weights of training model $f$ \\
\hspace*{0.1in}\textbf{for}{ epoch$\ = 1$ to $T$} \textbf{do} \\
\hspace*{0.2in}\textbf{for}{ mini batch $\{\boldsymbol{x}_i, y_i\}^{N}_{i=1}$ in $\mathcal{D}$} \textbf{do} \\ 
\hspace*{0.3in}\textbf{for}{ $i = 1$ to $N$} \textbf{do}\\
\hspace*{0.4in}\text{$\boldsymbol{x}'_{i}$ = PGD($\boldsymbol{x}_i$, $y_i$})\hspace*{2in}     \textcolor{gray}{\# PGD attack}\\ 
\hspace*{0.3in}\textbf{end for}\\
\hspace*{0.3in}$\theta \gets -\mu \frac{1}{n} \sum_{i=1}^{n}  \nabla_{\theta} \Big[ \mathcal{L}_{BSL}(f(\boldsymbol{x}'_i), y_i)$ \hspace*{0.9in}  \textcolor{gray}{\# Balanced softmax loss} \\
\hspace*{0.3in}\hspace*{1.1in}$ +  \alpha \cdot \mathcal{L}_{KD}(f(\boldsymbol{x}'_i), f_B(\boldsymbol{x}_i))\Big]$ \hspace*{0.43in} \textcolor{gray}{\# Self-distillation}\\
\hspace*{0.2in}\textbf{end for} \\
\hspace*{0.1in}\textbf{end for}

\end{algorithm*}

\section{Method}
We examine the impact of unbalanced datasets on performance disparity, particularly noting that this disparity becomes more pronounced during robust training compared to natural training.
To address this issue, we propose a simple yet effective self-distillation framework.

\subsection{Limitations of Existing Balanced Softmax Approaches}
Balanced softmax \citep{Wu_2021_CVPR, Yue_2024_CVPR} is a powerful method that effectively addresses the issue of tail-class robustness under adversarial training. 
These works demonstrate that applying Balanced Softmax improves tail-class robustness.
However, as shown in \Cref{tab:balance_versus_full},  while Balanced Softmax prevents the robustness of tail classes, it still falls short compared to naive PGD training on a balanced dataset in terms of tail-class robustness.
Moreover, our experiments in \Cref{tab:main_cifar10_res18}, \Cref{tab:main_cifar100_res18}, and \Cref{tab:main_tinyimg_res18} demonstrate that tail-class robustness under adversarial attacks remains notably lower than that of head classes.
This observation underscores the necessity of additional strategies to explicitly improve tail-class robustness.

\subsection{Training Self-Teacher to Guide Tail Class Robustness}

To address the limitations of Balanced Softmax and improve tail-class robustness, we construct a balanced sub-dataset \(D_B\) by up-sampling tail classes and down-sampling head classes. Specifically, suppose the number of samples of each class in $D$ is $n_1 < n_2 < \dots < n_C$, then we construct a new dataset where each class contains $\gamma \cdot n_1$ where $\gamma > 1$ is a hyperparameter of adjusting the number of $D_B$.
Using \(D_B\), we perform robust training with PGD, resulting in a self-teacher model that is more robust to tail classes compared to models trained on imbalanced datasets.
The balanced self-teacher transfers its tail robustness to the student model via adversarial knowledge distillation \cite{IGDM}, as detailed in \Cref{alg:algorithm}. Through this process, the proposed method addresses the insufficient tail-class robustness of Balanced Softmax, enhancing the model's robustness on tail classes while maintaining overall robustness.

\section{Experiments}

\subsection{Experiment Settings}

\noindent\textbf{Dataset.}
We conducted experiments using long-tailed distribution datasets: CIFAR-10-LT, CIFAR-100-LT \citep{krizhevsky2009learning}, and Tiny-ImageNet-LT \citep{le2015tiny}, with various imbalance ratios (IR), primarily set at 50 for CIFAR-10-LT, 10 for CIFAR-100-LT and Tiny-ImageNet-LT.
Random crop and random horizontal flip were applied, while other augmentations were not utilized unless specified.

\noindent\textbf{Training details.}
We employed ResNet-18 \citep{he2016deep} and WideResNet-34-10 \citep{zagoruyko2016wide} architectures for CIFAR-10/100-LT, and results for WideResNet-34-10 are included in the appendix.
For Tiny-ImageNet-LT, we employed PreActResNet-18 \citep{he2016identity}.
Initially, we trained a balanced self-teacher using the same model architecture for $30$ epochs using a batch size of $32$ with a balanced dataset, resampled by the original long-tailed dataset with $\gamma = IR/2$. 
In the main training phase, we trained for $100$ epoch using a batch size of $128$ with self-distillation from the balanced self-teacher.
We utilized SGD optimization to train both the balanced self-teacher and the main model, setting the learning rate to $0.1$ and weight decay to $5\times10^{-4}$.
We used an epsilon boundary of $8\slash255$, a commonly used setting in adversarial training, and employed a 10-step PGD attack during training.

\begin{table}[t]
\begin{center}
\caption{The clean accuracy and robustness for various algorithms using ResNet-18 on CIFAR-10-LT. T-Clean and T-PGD are clean and PGD-20 accuracy on tail class.}
\label{tab:main_cifar10_res18}
\setlength{\tabcolsep}{4.5pt}
 \renewcommand{\arraystretch}{1}
\begin{tabular}{lrrrrrrrrrr}
\toprule
 \multicolumn{1}{c}{\multirow{2}{*}{Method}} & \multicolumn{5}{c}{Best Checkpoint}  & \multicolumn{5}{c}{Last Checkpoint} \\
\cmidrule(r){2-6}
\cmidrule(r){7-11}
 & Clean  & PGD  & AA  & T-Clean & T-PGD&  Clean & PGD  & AA & T-Clean & T-PGD \\
\midrule
PGD-AT  & 52.71 & 29.30 & 27.57 & 12.7 & 1.0 & 56.39 & 26.98 & 25.81 & 20.8 & 2.2 \\
TRADES  & 45.79 & 28.66 & 27.01 & 6.7 & 0.8 & 47.10 & 28.00 & 26.45 & 6.4 & 0.6 \\
MART  & 44.03 & 29.36 & 27.59 & 5.0 & 0.5 & 47.33 & 28.08 & 26.55 & 10.9 & 1.0 \\
AWP  & 51.69 & 32.42 & 30.35 & 5.3 & 0.2 & 51.89 & 32.42& 30.35 & 10.9 & 0.6  \\
\midrule
RoBal   & 70.54 & 35.33 & 28.83 & 70.4 & 33.1 & \textbf{72.80} & 28.04 & 25.00 & 67.7 & 15.9   \\
REAT  & 68.34 & 35.98 & 32.45 & 69.5 & 29.1 & 68.32 & 28.67 & 26.68 & 55.7 & 11.6   \\
AT-BSL  &  68.43 & 35.87 & 32.27 & 63.2 & 22.0 & 67.60 & 29.40 & 27.46 & 50.1 & 8.7  \\
\textbf{Ours}  & \textbf{70.81} & \textbf{38.85} & \textbf{34.32} & \textbf{73.8} & \textbf{36.9} & 71.74 & \textbf{37.80} & \textbf{33.74} & \textbf{74.7} & \textbf{36.2}
\\
\bottomrule
\end{tabular}
\end{center}
\end{table}

\begin{table}[!t]
\begin{center}
\caption{The clean accuracy and robustness for various algorithms using ResNet-18 on CIFAR-100-LT. T-Clean and T-PGD are clean and PGD-20 accuracy on the tail class group.}
\label{tab:main_cifar100_res18}
\setlength{\tabcolsep}{4.5pt}
 \renewcommand{\arraystretch}{1}
\begin{tabular}{lrrrrrrrrrr}
\toprule
 \multicolumn{1}{c}{\multirow{2}{*}{Method}} & \multicolumn{5}{c}{Best Checkpoint}  & \multicolumn{5}{c}{Last Checkpoint} \\
\cmidrule(r){2-6}
\cmidrule(r){7-11}
 & Clean  & PGD  & AA  & T-Clean & T-PGD&  Clean & PGD  & AA & T-Clean & T-PGD \\
\midrule
PGD-AT  & 42.73 & 17.31 & 15.44 & 20.4 & 7.7 & 43.09 & 15.07 & 14.05 & 22.4 & 6.7 \\
TRADES  & 38.83 & 19.05 & 16.06 & 16.5 & 7.3 & 39.63 & 18.86 & 16.18 & 16.1 & 6.8\\
MART  & 38.57 & 19.90 & 17.10 & 16.6 & 7.7 & 40.31 & 17.07 & 15.21 & 19.5 & 7.2 \\
AWP  & 40.46 & 21.85 & 18.58 & 16.2 & 8.5 & 40.15 & 21.71 & 18.33 & 16.2 & 8.6  \\
\midrule
RoBal   & 44.27 & 19.67 & 16.78 & 18.4 & 8.0 & 46.46 & 16.28 & 14.73 & 23.3 & 6.7   \\
REAT   & 45.73 & 18.22 & 15.82 & 32.2  & 11.4  & 45.53 & 15.64 & 14.27 & 33.0 & 10.5   \\
AT-BSL & 45.38 & 18.04 & 15.73 & 33.1 & 12.4 & 45.48 & 15.36 & 14.07 & 31.5 & 9.1 \\
\textbf{Ours}  & \textbf{46.13} & \textbf{22.42} & \textbf{18.73} & \textbf{38.9} & \textbf{17.9} & \textbf{47.22} & \textbf{21.82} & \textbf{18.53} & \textbf{37.9} & \textbf{17.6} 
\\
\bottomrule
\end{tabular}
\end{center}
\end{table}

\begin{table}[!ht]
\begin{center}
\caption{The clean accuracy and robustness for various algorithms using PreActResNet-18 on Tiny-ImageNet-LT.}
\label{tab:main_tinyimg_res18}
\setlength{\tabcolsep}{4.5pt}
 \renewcommand{\arraystretch}{1}
\begin{tabular}{lrrrrrrrrrr}
\toprule
 \multicolumn{1}{c}{\multirow{2}{*}{Method}} & \multicolumn{5}{c}{Best Checkpoint}  & \multicolumn{5}{c}{Last Checkpoint} \\
\cmidrule(r){2-6}
\cmidrule(r){7-11}
 & Clean  & PGD  & AA  & T-Clean & T-PGD&  Clean & PGD  & AA & T-Clean & T-PGD \\
\midrule
PGD-AT  & 34.89 & 14.17 & 10.98 & 15.8 & 5.4 & 36.55 & 9.69 & 8.45 & 23.2 &5.0 \\
TRADES  & 33.76 & 13.71 & 10.00 & 15.4 & 6.4 & 32.97& 12.50 & 9.61 & 16.0 & 6.0\\
MART  & 31.15 & 15.45 & 11.94 & 14.4 & 6.0 & 32.91 & 12.42 & 10.32 & 17.8 & 7.6\\
AWP  & 32.28& 15.09& 11.27 & 14.2& 5.6& 32.13& 13.95& 11.10& 14.4&  6.6\\
\midrule
RoBal  & 35.25& 14.01& 10.44& 15.0& 4.4& 37.97& 10.51& 8.64& 21.8& 4.8\\
REAT  & 38.37 & 15.25 & 11.99 & 33.0& 12.6& 38.48 & 10.58 & 9.07 & 31.6& 8.8\\
AT-BSL  & 38.38 & 15.39 & 11.85 & 30.8 & 13.0 & 38.41 & 10.25 & 8.90 & 32.4 & 7.0\\
\textbf{Ours} & \textbf{38.44}& \textbf{17.02}& \textbf{12.57}& \textbf{36.6}& \textbf{16.0}& \textbf{49.37}& \textbf{14.09}& \textbf{11.15}& \textbf{33.8}&  \textbf{12.6}\\
\bottomrule
\end{tabular}
\end{center}
\vspace{-0.15 in}
\end{table}

\noindent\textbf{Comparison models.}
As comparison models, we utilized PGD-AT \citep{PGD}, TRADES \citep{TRADES}, MART \citep{MART}, and AWP \citep{2020_awp}, representing prominent approaches of AT.
Additionally, we followed RoBal \citep{Wu_2021_CVPR}, REAT \citep{li2023alleviating}, and AT-BSL \citep{Yue_2024_CVPR}, which focus on long-tailed adversarial training.
For long-tailed AT implementation, we meticulously followed the setting of existing methods such as learning rate, batch size, weight decay, etc.

\noindent\textbf{Evaluation.} 
Evaluation metrics included clean accuracy, accuracy under a 20-step PGD attack, and AutoAttack (AA) accuracy \citep{AutoAttack}.
Additionally, we assessed clean and 20-step PGD attack accuracy specifically for tail classes, denoted as T-Clean and T-PGD, respectively.
In CIFAR-10-LT, the performance evaluation focused on the last class, while CIFAR-100-LT and Tiny-ImageNet-LT evaluated the performance of the tail 10 and 20 classes out of 100 and 200 classes, respectively.
We measured performance at both the best and last epoch based on the accuracy under the 20-step PGD attack.

\subsection{Main Results}

We demonstrated excellent performance across all datasets, including CIFAR-10-LT, CIFAR-100-LT, and Tiny-ImageNet-LT in \Cref{tab:main_cifar10_res18}, \Cref{tab:main_cifar100_res18}, and \Cref{tab:main_tinyimg_res18}, respectively.
The experimental results on the WideResNet-34-10 architecture can be found in  \Cref{tab:main_cifar10_wrn}and \Cref{tab:main_cifar100_wrn} in the appendix.
Particularly noteworthy is the substantial improvement in performance for tail classes.
When adversarial training methods such as PGD-AT, TRADES, MART, and AWP are naively applied to long-tailed datasets, overall performance remains reasonable compared to Robal, REAT, AT-BSL, but performance for the tail classes notably suffers.
For instance, while AWP exhibits superior performance compared to RoBal, the clean accuracy and robust accuracy for tail classes are significantly low.
Long-tailed adversarial training methods such as RoBal, REAT, and AT-BSL show considerable improvement in tail class performance compared to other adversarial training methods. 
However, when compared to the performance of the entire class, it is still evident that the performance remains imbalanced.
In contrast, our method shows significant improvement in the performance of the tail classes, resulting in minimal difference compared to the performance of the entire classes.
Additionally, we achieved overall better performance than the baseline at both the best and last checkpoints.

\subsection{Ablation}
In this section, we conduct further experiments to corroborate our main contribution.

\subsubsection{Augmentation}
Following the inclusion of various augmentation experiments outlined in AT-BSL, we conducted experiments applying RandAugment (RA) \citep{cubuk2020randaugment} and AutoAugment (AuA) \citep{cubuk2019autoaugment} in \Cref{tab:ablation_Aug}. 
While applying augmentation led to overall performance improvements, the best results were achieved when augmentation was applied to our method.
Our method consistently outperformed baselines on robustness with augmentation setting including tail class performance with augmentation.

\begin{table*}[t]
\begin{center}
\caption{The clean accuracy and robustness with augmentation using ResNet-18 on CIFAR-100-LT. T-Clean and T-PGD are clean and 20-step PGD accuracy on the tail class group.}
\label{tab:ablation_Aug}
\setlength{\tabcolsep}{3.5pt}
 \renewcommand{\arraystretch}{1}
\begin{tabular}{lrrrrrrrrrr}
\toprule
\multicolumn{1}{c}{\multirow{2}{*}{Method}} & \multicolumn{5}{c}{Best Checkpoint}  & \multicolumn{5}{c}{Last Checkpoint} \\
\cmidrule(r){2-6}
\cmidrule(r){7-11}
& Clean  & PGD  & AA  & T-Clean & T-PGD&  Clean & PGD  & AA & T-Clean & T-PGD \\
\midrule
Robal   & 44.27 & 19.67 & 16.78 & 18.4 & 8.0 & 46.46 & 16.28 & 14.73 & 23.3 & 6.7   \\
Robal-RA 
& 44.64 & 20.11 & 17.02 & 15.8 & 7.7 & 47.62 &  18.63 & 16.05 & 19.2 & 7.4
\\
Robal-AuA 
& 45.87 & 20.24 & 17.05 & 17.4 & 6.7 & 47.42 & 19.32 & 16.30 & 18.2 &7.4
\\
\midrule

Reat  & 45.38 & 18.04 & 15.73 & 33.1 & 12.4 & 45.48 & 15.36 & 14.07 & 31.5 & 9.1 \\
Reat-RA 
& 46.94 & 21.71 & 18.02 & 33.3 & 14.9  & 50.41 & 20.33 & 17.46 & 36.6 & 14.7
\\
Reat-AuA

& 47.86  & 23.09 & 19.43 & 34.0 & 16.7 & 50.56 & 22.20 & 18.60 & 36.5 & 16.5
\\
\midrule
AT-BSL   & 45.38 & 18.04 & 15.73 & 33.1 & 12.4 & 45.48 & 15.36 & 14.07 & 31.5 & 9.1 \\
AT-BSL-RA 
& 48.38 & 22.18 & 18.58 & 34.7 & 16.8& 50.33 & 20.29 & 17.42 & 37.1 & 14.6
\\
AT-BSL-AuA 
& 47.30 & 22.78 & 18.66& 34.2 & 16.4 & 50.57 & 21.98 & 18.45 & 36.8 & 16.3
\\
\midrule
\textbf{Ours}   
&  46.13 & 22.42 & 18.73 &  \textbf{38.9} & 17.9 & 47.22 & 21.82 & 18.53 & 37.9 & 17.6
\\
\textbf{Ours-RA} 
& 48.78 & 23.58 & 19.30 & 34.9 & 17.2 & \textbf{50.98}& 22.43& 18.80& \textbf{38.2}& 16.9\\
\textbf{Ours-AuA} & \textbf{50.14} & \textbf{24.60} & \textbf{20.08} & 36.4& \textbf{18.0} & 50.46 & \textbf{24.32} & \textbf{20.28} & 37.6 & \textbf{18.5}
\\

\bottomrule
\end{tabular}
\end{center}
\end{table*}

\begin{table*}[t]
\begin{center}
\caption{The clean accuracy and robustness with different Imbalance Ratio(IR) using ResNet-18 on CIFAR-100-LT. T-Clean and T-PGD are clean and 20-step of PGD accuracy on the tail class group.}
\label{tab:ablation_IR}
\setlength{\tabcolsep}{4pt}
 \renewcommand{\arraystretch}{1}
\begin{tabular}{llrrrrrrrrrrr}
\toprule
\multicolumn{1}{c}{\multirow{2}{*}{IR}} & \multicolumn{1}{c}{\multirow{2}{*}{Method}} & \multicolumn{5}{c}{Best Checkpoint}  & \multicolumn{5}{c}{Last Checkpoint} \\
\cmidrule(r){3-7}
\cmidrule(r){8-12}
& & Clean  & PGD  & AA  & T-Clean & T-PGD&  Clean & PGD  & AA & T-Clean & T-PGD \\

\midrule
\multicolumn{1}{c}{\multirow{4}{*}{50}} & RoBal & 33.52 & 14.56 & 12.27 & 1.9 & 0.9   
& 34.81& 12.16& 11.02& 4.6& 1.7
\\
& REAT & 26.62 & 13.73 & 10.64 & 9.2 & 4.1   
& 36.51& 12.28& 11.15& 19.1 & 4.4
\\
& AT-BSL & 30.06 & 13.80 & 10.91 & 10.3 & 4.2  
& 36.46 & 12.07 & 11.21& 17.8 & 4.4
\\
& \textbf{Ours} & \textbf{38.09} & \textbf{16.65} & \textbf{13.58} & \textbf{14.8} & \textbf{5.2}
& \textbf{38.56}& \textbf{16.08}& \textbf{13.52}&  \textbf{19.2} & \textbf{5.8}\\


\midrule
\multicolumn{1}{c}{\multirow{4}{*}{20}} & RoBal  & 40.08 & 16.91 & 14.28 & 12.1 & 5.4
& 41.28& 13.96& 12.70& 15.5& 5.1
\\
& REAT  & 33.17 & 15.82 & 13.01& 20.0& 9.2
& 41.73& 13.79& 12.58& 29.8& 7.7
\\
& AT-BSL  & 41.70 & 15.51 & 13.62 & 30.7 & 10.4
& 41.41 & 13.49 & 12.48 & 27.1 & 7.9 
\\
& \textbf{Ours}  & \textbf{42.54}& \textbf{19.66}& \textbf{16.36}& \textbf{33.2}& \textbf{13.3}& \textbf{42.64}& \textbf{19.24}& \textbf{15.97}& \textbf{32.5}& \textbf{13.7}\\
\midrule
\multicolumn{1}{c}{\multirow{4}{*}{10}} & RoBal & 44.27 & 19.67 & 16.26 & 16.7 & 7.6
& 46.46 & 16.28 & 14.73 & 23.3 & 6.7   
\\
& REAT & 45.73 & 18.22 & 15.82 & 34.4 & 12.2
& 45.53 & 15.64 & 14.27 & 33.0 & 10.5   
\\
& AT-BSL & 45.38 & 18.04  & 15.73 & 33.1 & 12.4 
& 45.48 & 15.36 & 14.07 & 31.5 & 9.1 
\\
& \textbf{Ours} & \textbf{47.22} & \textbf{21.82} & \textbf{18.53} & \textbf{37.9} & \textbf{17.6} & \textbf{47.22} & \textbf{21.82} & \textbf{18.53} & \textbf{37.9} & \textbf{17.6} 
\\

\midrule
\multicolumn{1}{c}{\multirow{4}{*}{5}} & RoBal & 49.49& 21.66 & 18.59 & 26.1& 11.9
& 51.56& 18.15& 16.58& 34.7& 9.9
\\
& REAT &49.48  & 21.95 & 18.98 & 39.3 & 18.3 & 49.76& 18.19& 16.65& 40.7& 14.1
\\
& AT-BSL & 49.75& 21.53 & 18.65 & 42.3 & 18.2 & 49.41 & 18.12 & 16.56 & 40.0 & 13.7 
\\
& \textbf{Ours} & \textbf{50.77} & \textbf{24.13}& \textbf{20.10}& \textbf{44.0} & \textbf{20.9}
& \textbf{51.92}& \textbf{25.00}& \textbf{21.14}& \textbf{43.8}& \textbf{20.4}
\\



\bottomrule
\end{tabular}
\end{center}
\end{table*}

\subsubsection{Different Imbalance Ratio}
In \Cref{tab:ablation_IR}, we conducted experiments using different imbalance ratios (IR).
As the IR increases, the number of tail classes decreases, leading to decreased robustness.
In all cases, our method outperforms the baseline in terms of both overall and tail robustness.
This indicates that our proposed framework generally performs well across different IR settings.

\subsubsection{Effect of Balanced Subset}

To evaluate the effect of the balanced subset $D_B$, we compare the performance of models trained with PGD-AT on $D_B$ against Robal, Reat, and BSL trained on $D$, which incorporate techniques like balanced softmax to address long-tailed distributions.
For simplicity, we denote $\text{PGD-AT}_{D_B}$ as the model trained with PGD-AT on the balanced subset $D_B$.
\begin{wrapfigure}{r}{0.45\textwidth}
  \begin{center}
    \includegraphics[width=0.4\textwidth]{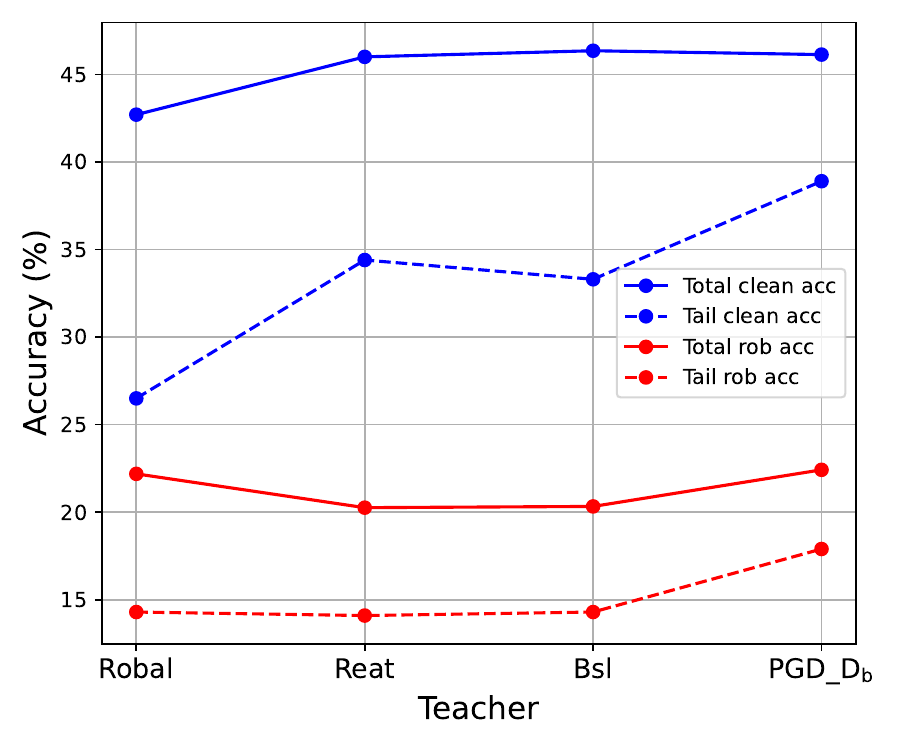}
  \end{center}
  \caption{Clean and robust accuracy according to different teachers. Robal, Reat, and Bsl were trained with 100 epochs, while $\text{PGD-AT}_{D_B}$ used a teacher trained with 30 epochs.}
  \label{fig:diff_teacher}
\end{wrapfigure}
As shown in \Cref{tab:balance_versus_full}, while the PGD-AT model trained on $D_B$ achieved the lowest overall performance, it demonstrated the best results for tail classes, T-Clean and T-PGD.
This suggests that even with fewer training epochs, the balanced subset is effective for improving performance on tail classes.
In \Cref{fig:diff_teacher}, we apply the same methods with different teachers where the performance is summarized in \Cref{tab:balance_versus_full}.
Interestingly, the best results were achieved when we utilize $\text{PGD-AT}_{D_B}$ as the teacher model, despite having the lowest overall performance.
Specifically, the performance on the tail classes highlights the effectiveness of the teacher trained on the balanced subset, as it demonstrates superior performance on the tail class compared to other methods. This underscores the utility of the balanced subset in improving tail class performance.
Additionally, although we trained the teacher for self-distillation using a simple method, PGD-AT, developing a more effective teacher remains an area for future work.

\begin{table*}[t]
\begin{center}
\caption{The clean accuracy and robustness using ResNet-18 on CIFAR-100-LT. T-Clean and T-PGD represent clean and 20-step PGD accuracy on a tail class.  }
\label{tab:balance_versus_full}
\setlength{\tabcolsep}{7pt}
 \renewcommand{\arraystretch}{1}
\begin{tabular}{l|l|rrrr}
\toprule
 Dataset & Method & Clean & PGD & T-Clean & T-PGD \\
\midrule
\multirow{3}{*}{$D$}  
                    & Robal   & 44.27 &  \textbf{19.67} & 18.4 & 8.0 \\
                    & Reat  & \textbf{45.73} &  18.22 & 32.2 &  11.4 \\
                    & BSL  & 45.38 &  18.04 & 33.1 & 12.4\\
\midrule
$D_B$              & PGD-AT   & 35.71 & 14.94 & \textbf{34.6}  & \textbf{14.1} \\
\bottomrule
\end{tabular}
\end{center}
\end{table*}

\section{Conclusion}

Building on the observation that adversarial training methods inherently struggle with tail classes, we propose effective strategies to address the lower performance on these classes. We first train a balanced self-teacher and subsequently perform knowledge distillation from this self-teacher.
This approach leads to significant improvements in long-tailed adversarial training, enhancing both overall robustness and tail class 
robustness.

\noindent \textbf{Discussion}
It is well known that adversarial training varies in difficulty across classes, and performance also differs by class. This presents a fairness issue, indicating that in robustness, not only the number of data points but also the intrinsic difficulty of each class plays a role. While this paper focuses solely on data quantity, designing robust models that account for class-level fairness remains an area for future work.

\bibliography{AT}
\bibliographystyle{iclr2025_conference}
\appendix
\section{Theoritical Proof}
\label{sec:proof}

\begin{lemma}
\label{lemma:identical_weight_natural}
    Given the data distribution $\mathcal{S}$, an optimal natural classifier that minimizes the overall standard error has optimal weight that satisfies $w_1 = w_2 = \cdots = w_n.$
\end{lemma}

\begin{proof}
    Let's assume, for the sake of contradiction, that the optimal weights do not satisfy the given condition. In other words, for some $i\neq j$ and $i, j \in \{1, 2, \cdots, n\}$, we assume if there exist $w_i < w_j$. Then, we obtain the following standard error
    \begin{align}
        \mathcal{R}_{nat}(f) &= \Pr(y=-1)\cdot\Pr \left(\sum^n_{k\neq i, k\neq j} w_k \mathcal{N}(-\eta, 1) + b + w_i\mathcal{N}(-\eta, 1)+ w_j\mathcal{N}(-\eta, 1) > 0\right)\nonumber\\
        &+\Pr(y=+1)\cdot\Pr \left(\sum^n_{k\neq i, k\neq j} w_k \mathcal{N}(+\eta, 1) + b + w_i\mathcal{N}(+\eta, 1)+ w_j\mathcal{N}(+\eta, 1) < 0\right)
    \end{align}
    However, if we define a new classifier $f'$, which has the same weight vector as classfier $f$ but uses $w_j$ to replace $w_i$. The resulting standard error for the new classifier $f'$ can be obtained as follows
    \begin{align}
        \mathcal{R}_{nat}(f') &= \Pr(y=-1)\cdot\Pr \left(\sum^n_{k\neq i, k\neq j} w_k \mathcal{N}(-\eta, 1) + b + w_j\mathcal{N}(-\eta, 1)+ w_j\mathcal{N}(-\eta, 1) > 0\right)\nonumber\\
        &+\Pr(y=+1)\cdot\Pr \left(\sum^n_{k\neq i, k\neq j} w_k \mathcal{N}(+\eta, 1) + b + w_j\mathcal{N}(+\eta, 1)+ w_j\mathcal{N}(+\eta, 1) < 0\right)
    \end{align}
Given that $w_i < w_j$, the $f'$ has a smaller error than $f$, which contradicts the assumption that $f$ is the optimal classifier with the least error.
\end{proof}

\begin{lemma}
\label{lemma:identical_weight_robust}
    Given the data distribution $\mathcal{S}$, an optimal robust classifier that minimizes the robust error has optimal weight that satisfies $w_1 = w_2 = \cdots = w_n$.  
\end{lemma}
Similar to the \Cref{lemma:identical_weight_natural}, it can be easily proved with the same argument.

\subsection{Proof of \Cref{theorem:1}}
\begin{proof}
    By the \Cref{lemma:identical_weight_natural}, the optimal classifier $f_{nat}$ for standard error has optimal weight of $w_1 = w_2 = \cdots = w_n.$
    For simplicity, we assume l2-norm of $\boldsymbol{w}$ = 1, \ie, $\boldsymbol{w} = (1/\sqrt{n}, 1/\sqrt{n}, \dots, 1/\sqrt{n}).$ following existing works \cite{FairAT, DAFA}.
    Then, the standard errors of $f_{nat}$ can be formulated as follows.
    \begin{align}
         \mathcal{R}_{nat}(f_{nat}) &= \Pr(y=+1)\cdot\mathcal{R}^{+1}_{nat}(f_{nat}) + \Pr(y=-1)\cdot\mathcal{R}^{-1}_{nat}(f_{nat})\nonumber \\
         &=\frac{r}{r+1} \Pr \left(f(\boldsymbol{x}) \neq y | y = +1\right) + \frac{1}{r+1}\Pr \left(f(\boldsymbol{x}) \neq y | y = -1\right)\nonumber\\
         &=\frac{r}{r+1} \Pr \left(\sum^n_{k=1} \frac{1}{\sqrt{n}} \mathcal{N}(+\eta, 1) + b  < 0\right) + \frac{1}{r+1}\Pr \left(\sum^n_{k=1} \frac{1}{\sqrt{n}} \mathcal{N}(-\eta, 1) + b  > 0\right)\nonumber\\ 
         &=\frac{r}{r+1}\cdot\Phi(-\sqrt{n}\eta -b ) +  \frac{1}{r+1}\cdot\Phi(-\sqrt{n}\eta + b)
    \end{align}
    Here, $\Phi$ represents the cumulative distribution function of the standard normal distribution.
    To determine the optimal value of $b$, we solve the equation $d \mathcal{R}_{nat}(f_{nat})/db = 0$.
    \begin{align}
        \frac{d \mathcal{R}_{nat}(f_{nat})}{d b} = -\frac{r}{r+1}\cdot\phi(-\sqrt{n}\eta -b ) +  \frac{1}{r+1}\cdot\phi(-\sqrt{n}\eta + b) &= 0\nonumber\\
        -r\cdot\phi(-\sqrt{n}\eta -b ) +  \phi(-\sqrt{n}\eta + b) &= 0\nonumber\\
        -r\cdot\exp\left(-\frac{1}{2}(-\sqrt{n}\eta -b )^2\right) +  \exp\left(-\frac{1}{2}(-\sqrt{n}\eta +b )^2\right) &= 0
    \end{align}
    Here, $\phi$ represents the standard normal distribution function.
    Therefore, the optimal $b^*_{nat}$ for natural classifier is follows,
    \begin{equation}
        b^*_{nat} = \frac{\ln r}{2\sqrt{n}\eta}.
    \end{equation}
    By using the optimal natural classifier, the standard error of the tail class can be formulated as follows,
    \begin{equation}
        \mathcal{R}^{-1}_{nat}(f^*_{nat}) = \Phi\left(-\sqrt{n}\eta + \frac{\ln r}{2\sqrt{n}\eta}\right).
    \end{equation}
    Then, the robust error of the tail class with optimal natural classifier can be calculated as follows,
    \begin{align}
        \mathcal{R}^{-1}_{rob}(f^*_{nat}) &= \Pr(\exists \boldsymbol{\delta} \text{ with } \| \boldsymbol{\delta} \|_\infty \leq \epsilon \text{ s.t. } f^*_{nat}(\boldsymbol{x} + \boldsymbol{\delta}) \neq y | y=-1) \nonumber\\
        &=\Pr \left(\sum^n_{k=1} \frac{1}{\sqrt{n}}(x_i +\epsilon) + b^*_{nat}  > 0\right) \nonumber\\
        &=\Pr \left(\sum^n_{k=1} \frac{1}{\sqrt{n}} \mathcal{N}(-\eta +\epsilon, 1) + b^*_{nat}  > 0\right) \nonumber\\
        &=\Pr \left( \mathcal{N}(0, 1) < -\sqrt{n}(\eta - \epsilon) + b^*_{nat} \right) \nonumber\\
        &= \Phi\left(-\sqrt{n}(\eta-\epsilon) + \frac{\ln r}{2\sqrt{n}\eta}\right).
    \end{align}
    Similarly, based on the \Cref{lemma:identical_weight_robust}, the optimal classifier $f_{rob}$ for robust error has optimal weight of $w_1 = w_2 = \cdots = w_n = 1/\sqrt{n}$.
    Therefore, the robust errors of $f_{rob}$ can be formulated as follows with adversarial noise $\epsilon$ satisfying $0<\epsilon<\eta$
    \begin{align}
\mathcal{R}_{rob}(f_{rob}) &= \Pr(y=+1)\cdot\mathcal{R}^{+1}_{rob}(f_{rob}) + \Pr(y=-1)\cdot\mathcal{R}^{-1}_{rob}(f_{rob})\nonumber \\
&=\frac{r}{r+1}\cdot\Pr \left(\sum^n_{k=1} \frac{1}{\sqrt{n}} \mathcal{N}(+\eta -\epsilon, 1) + b  < 0\right) \nonumber\\
      &\quad+\frac{1}{r+1} \cdot\Pr \left(\sum^n_{k=1} \frac{1}{\sqrt{n}} \mathcal{N}(-\eta +\epsilon, 1) + b  > 0\right)\nonumber\\
&= \frac{r}{r+1} \cdot \Pr \left( \mathcal{N}(0, 1) < -\sqrt{n}(\eta - \epsilon) - b \right) \nonumber\\
      &\quad + \frac{1}{r+1} \cdot \Pr \left( \mathcal{N}(0, 1) < -\sqrt{n}(\eta - \epsilon) + b \right)\nonumber \\
         &=\frac{r}{r+1}\cdot\Phi(-\sqrt{n}(\eta-\epsilon) -b ) +  \frac{1}{r+1}\cdot\Phi(-\sqrt{n}(\eta-\epsilon) + b)
    \end{align}
    To determine the optimal value of $b$, we solve the equation $d \mathcal{R}_{rob}(f_{rob})/db = 0$.
    \begin{align}
        \frac{d \mathcal{R}_{rob}(f_{rob})}{d b} = -\frac{r}{r+1}\cdot\phi(-\sqrt{n}(\eta-\epsilon) -b ) +  \frac{1}{r+1}\cdot\phi(-\sqrt{n}(\eta-\epsilon) + b) &= 0\nonumber\\
        -r\cdot\phi(-\sqrt{n}(\eta-\epsilon) -b ) +  \phi(-\sqrt{n}(\eta-\epsilon) + b) &= 0\nonumber\\
        -r\cdot\exp\left(-\frac{1}{2}(-\sqrt{n}(\eta-\epsilon) -b )^2\right) +  \exp\left(-\frac{1}{2}(-\sqrt{n}(\eta-\epsilon) +b )^2\right) &= 0
    \end{align}
    Therefore, the optimal $b^*_{rob}$ for robust classifier is follows,
    \begin{equation}
        b^*_{rob} = \frac{\ln r}{2\sqrt{n}(\eta-\epsilon)}.
    \end{equation}
    Then, the standard and robust error of the tail class with optimal robust classfier can be formulated as follows,
    \begin{align}
        \mathcal{R}^{-1}_{nat}(f^*_{rob}) &= \Phi\left(-\sqrt{n}\eta + \frac{\ln r}{2\sqrt{n}(\eta-\epsilon)}\right), \\
        \mathcal{R}^{-1}_{rob}(f^*_{rob}) &= \Phi\left(-\sqrt{n}(\eta-\epsilon) + \frac{\ln r}{2\sqrt{n}(\eta-\epsilon)}\right).
    \end{align}
\end{proof}

\section{Additional Experiments}

\subsection{Experiments on another architecture.}

We conducted experiments using WideResNet-34-10. Similar to the results of ResNet-18 in the main paper, our method consistently demonstrated superior performance.
Notably, on both CIFAR-10-LT and CIFAR-100-LT datasets, significant performance improvements were observed in both T-Clean and T-PGD settings.
While RoBal exhibited a marginally higher clean accuracy in the case of the best checkpoint on CIFAR-10-LT, the difference compared to our method is negligible. However, our method achieved approximately a 5\% point improvement in robust accuracy against auto attack on CIFAR-10-LT.
In the CIFAR-100-LT dataset, our method demonstrated the best performance in terms of both clean accuracy and robustness across all classes.
Additionally, the improvements in T-Clean and T-PGD demonstrate that our method is more suitable for handling long-tail distributions.

\begin{table}[h]
\begin{center}
\caption{The clean accuracy and robustness for various algorithms using WideResNet-34-10 on CIFAR-10-LT.}
\label{tab:main_cifar10_wrn}
\setlength{\tabcolsep}{4pt}
 \renewcommand{\arraystretch}{1}
\begin{tabular}{lrrrrrrrrrr}
\toprule
 \multicolumn{1}{c}{\multirow{2}{*}{Method}} & \multicolumn{5}{c}{Best Checkpoint}  & \multicolumn{5}{c}{Last Checkpoint} \\
\cmidrule(r){2-6}
\cmidrule(r){7-11}
 & Clean  & PGD  & AA  & T-Clean & T-PGD&  Clean & PGD  & AA & T-Clean & T-PGD \\
\midrule
PGD-AT  & 58.86 & 30.57 & 29.43 & 18.5 & 2.1 & 59.10 & 26.3 & 25.66 & 19.0 & 1.9\\
TRADES  & 51.93 & 30.45 & 29.20 & 4.5 & 0.3 & 55.36 & 27.62 & 26.99 & 19.2 & 2.9\\
MART & 48.92 & 31.45 & 29.85 & 9.5 & 0.9 & 54.81 & 27.25 & 26.29 & 23.1 & 2.0\\
AWP  & 51.69 & 32.42 & 30.35 & 5.3 & 0.2 & 51.89 & 29.19 & 27.45 & 10.9& 0.6  \\
\midrule
RoBal   & \textbf{74.46}& 32.82& 25.72& 71.5& 22.8& 70.03& 24.74& 23.09& 50.6& 5.7\\
REAT   & 73.16& 33.45& 28.71 & 66.4& 20.8& 64.11 & 25.90 & 25.00 & 31.7 & 3.6  \\
AT-BSL  & 73.23 & 35.08 & 32.26 & 66.4 & 18.9 & 66.23 & 26.87 & 25.98 & 40.6 & 4.3\\
\textbf{Ours} & 73.97& \textbf{39.25}& \textbf{35.97}& \textbf{74.3}& \textbf{33.7}& \textbf{72.38}& \textbf{31.15} & \textbf{29.10} & \textbf{60.4}& \textbf{12.7}\\
\bottomrule
\end{tabular}
\end{center}
\vspace{-0.15 in}
\end{table}

\begin{table}[h]
\begin{center}
\caption{The clean accuracy and robustness for various algorithms using WideResNet-34-10 on CIFAR-100-LT.}
\label{tab:main_cifar100_wrn}
\setlength{\tabcolsep}{4pt}
 \renewcommand{\arraystretch}{1}
\begin{tabular}{lrrrrrrrrrr}
\toprule
 \multicolumn{1}{c}{\multirow{2}{*}{Method}} & \multicolumn{5}{c}{Best Checkpoint}  & \multicolumn{5}{c}{Last Checkpoint} \\
\cmidrule(r){2-6}
\cmidrule(r){7-11}
 & Clean  & PGD  & AA  & T-Clean & T-PGD&  Clean & PGD  & AA & T-Clean & T-PGD \\
\midrule
PGD-AT  & 47.48 & 19.36 & 17.79 & 25.4 & 8.7 & 46.09 & 16.51 & 15.67 & 24.3 & 7.2\\
TRADES  & 42.67 & 20.89 & 18.42 & 18.3 & 6.9 & 43.99  & 18.53 & 17.51 & 19.9 & 7.4 \\
MART  & 41.54 & 21.52 & 18.83 & 19.2 & 9.2 & 43.08 & 17.00 & 15.84 & 22.8 & 7.8 \\
AWP  & 45.53 & 23.23 & 19.92 & 20.4 & 7.9 & 47.05 & 21.97 & 19.21 & 23.1& 8.8  \\
\midrule
RoBal   & 49.06 & 18.23 & 16.79 & 27.6 & 9.4 & 46.92 & 15.48 & 14.69 & 28.0 & 6.8 \\
REAT   & 49.06 & 20.00 & 18.08 & 34.4 & 12.2 & 47.65 & 16.95 & 15.60 & 33.6 & 9.8\\
AT-BSL & 50.05 & 18.96 & 17.10 & 38.3 & 13.3 & 47.95 & 16.40 & 15.31 & 32.2 & 9.5\\
\textbf{Ours} & \textbf{50.55} & \textbf{23.43} & \textbf{20.16} & \textbf{38.4} & \textbf{19.5} & \textbf{50.87} & \textbf{22.21} & \textbf{19.44} & \textbf{42.5} & \textbf{18.4}\\
\bottomrule
\end{tabular}
\end{center}
\vspace{-0.15 in}
\end{table}

\subsection{Additional experiment of more training epochs}

Since we employed additional training epochs for self-distillation, we also trained the baselines with more epochs and summarized the results in \Cref{tab:more_epochs}.
The results showed that increasing the training epochs for the baselines did not lead to performance improvements; in REAT, performance even declined when more training epochs were utilized.
As a result, it is clear that the efficacy of our approach is not solely a consequence of increasing the number of training epochs.

\begin{table}[t]
\begin{center}
\caption{The clean accuracy and robustness for various algorithms using ResNet-18 on CIFAR-100-LT. T-Clean and T-PGD are clean and PGD-20 accuracy on the tail class group.}
\label{tab:more_epochs}
\setlength{\tabcolsep}{2pt}
 \renewcommand{\arraystretch}{1}
\begin{tabular}{lrrrrrrrrrr}
\toprule
 \multicolumn{1}{c}{\multirow{2}{*}{Method}} & \multicolumn{5}{c}{Best Checkpoint}  & \multicolumn{5}{c}{Last Checkpoint} \\
\cmidrule(r){2-6}
\cmidrule(r){7-11}
 & Clean  & PGD  & AA  & T-Clean & T-PGD&  Clean & PGD  & AA & T-Clean & T-PGD \\
\midrule
RoBal (100 epochs)   & 44.27 & 19.67 & 16.78 & 18.4 & 8.0 & 46.46 & 16.28 & 14.73 & 23.3 & 6.7   \\
RoBal (200 epochs)   & 44.20 & 19.70 & 17.01 & 17.5 & 8.1 & 45.60 & 15.06 & 14.03 & 24.1 & 6.7   \\
\midrule
REAT (100 epochs)  & 45.73 & 18.22 & 15.82 & 32.2  & 11.4  & 45.53 & 15.64 & 14.27 & 33.0 & 10.5   \\
REAT (200 epochs)  & 44.67 & 16.48 & 14.53 & 29.6  & 10.3  & 44.80 & 14.86 & 13.64 & 28.9 & 7.9   \\
\midrule
AT-BSL (100 epochs) & 45.38 & 18.04 & 15.73 & 33.1 & 12.4 & 45.48 & 15.36 & 14.07 & 31.5 & 9.1 \\
AT-BSL (200 epochs) & 45.01  & 17.19  & 14.56 & 28.9 & 9.7  & 44.04 & 14.23  & 13.23 & 28.5 & 7.2  \\
\midrule
\textbf{Ours}  & \textbf{46.13} & \textbf{22.42} & \textbf{18.73} & \textbf{38.9} & \textbf{17.9} & \textbf{47.22} & \textbf{21.82} & \textbf{18.53} & \textbf{37.9} & \textbf{17.6} 
\\
\bottomrule

\end{tabular}
\end{center}
\end{table}

\begin{figure}[H]
 \centering
 \begin{subfigure}[b]{0.47\textwidth}
   \centering
   \includegraphics[width=\textwidth]{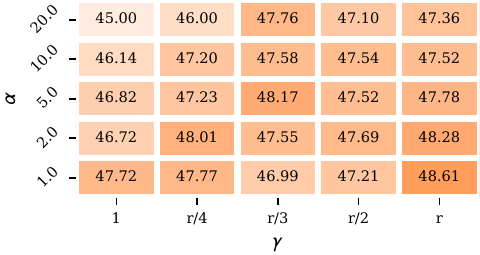}
   \caption{Clean accuracy}
 \end{subfigure}
 \hfill
 \begin{subfigure}[b]{0.47\textwidth}
   \centering
   \includegraphics[width=\textwidth]{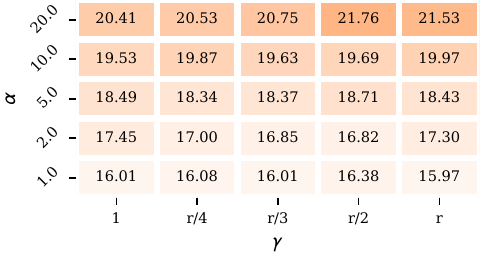}
   \caption{Robust accuracy}
 \end{subfigure}
 \caption{Hyperparmeter sensitivity of entire class }
 \label{fig:hyper_sensitivity_full}
\end{figure}

\begin{figure}[H]
 \centering
 \begin{subfigure}[b]{0.47\textwidth}
   \centering
   \includegraphics[width=\textwidth]{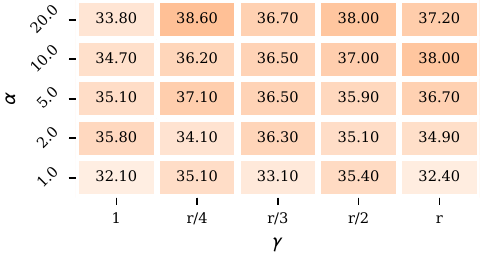}
   \caption{Clean accuracy}
 \end{subfigure}
 \hfill
 \begin{subfigure}[b]{0.47\textwidth}
   \centering
   \includegraphics[width=\textwidth]{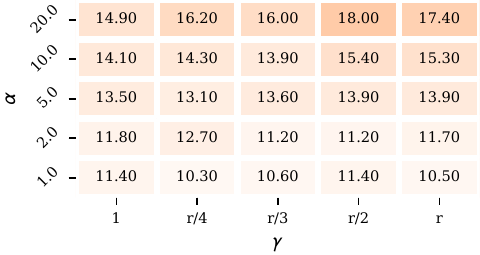}
   \caption{Robust accuracy}
 \end{subfigure}
 \caption{Hyperparmeter sensitivity of tail class }
  \label{fig:hyper_sensitivity_tail}
\end{figure}

\subsection{Sensitivity of Hyperparameter}
In \Cref{fig:hyper_sensitivity_full}, we experiment with the sensitivity of the self-distillation weight parameter, $\alpha$, and the sampling rate, $\gamma$, where  $r$ is an imbalance ratio between the class with the largest number of samples and the class with the smallest number of samples.
We can see that as $\alpha$ increases, robustness improves, but clean accuracy slightly decreases.
This indicates a trade-off between robustness and clean accuracy, which is expected given the use of adversarial distillation techniques.
In the case of $\gamma$, it did not significantly impact performance. However, when $\gamma$ is larger—meaning more samples are used to train the self-teacher—both clean accuracy and robustness showed improvement.

In \Cref{fig:hyper_sensitivity_tail}, we plot the tail class performance. 
In this case, we observed that as $\alpha$ increases, \ie, as the weight of the loss for knowledge distillation from the balanced self-teacher increases, the clean and robust performance of the tail class improves.
The sensitivity to $\gamma$ was not significant.

\subsection{Variance of multiple runs}
\begin{table}[h]
\begin{center}
\caption{The clean accuracy and robustness for various algorithms using ResNet18 on CIFAR-100-LT.}
\label{tab:mutiple_runs}
\setlength{\tabcolsep}{10pt}
\renewcommand{\arraystretch}{1}
\begin{tabular}{ccccc}
\toprule
Runs & Clean & PGD & T-Clean & T-PGD \\
\midrule
1 & 46.13 & 22.42 &  38.9 & 17.9 \\
2 & 46.57 & 22.23 &  37.5 & 17.7 \\
3 &  46.47 & 22.27 &  37.8 & 17.8 \\
4 & 46.59 & 22.12 & 38.8 & 17.8 \\
5 & 46.01 & 22.52 &  38.9 & 18.0 \\

\midrule
Average & 46.35 & 22.31 & 38.38 & 17.84 \\
Standard deviation & 0.27 & 0.15 & 0.68 & 0.11  \\
\bottomrule
\end{tabular}
\end{center}
\vspace{-0.15 in}
\end{table}

In \Cref{tab:mutiple_runs}, we conducted five experiments and computed the mean and standard deviation to assess the impact of randomness. The results show that the standard deviations are relatively small, indicating that the model's performance is consistent across different runs. This suggests that the observed improvements are robust and not significantly influenced by random fluctuations in the training process.

\subsection{Class-wise robustness.}
\begin{figure}[h]
 \centering

   \includegraphics[width=0.75\textwidth]{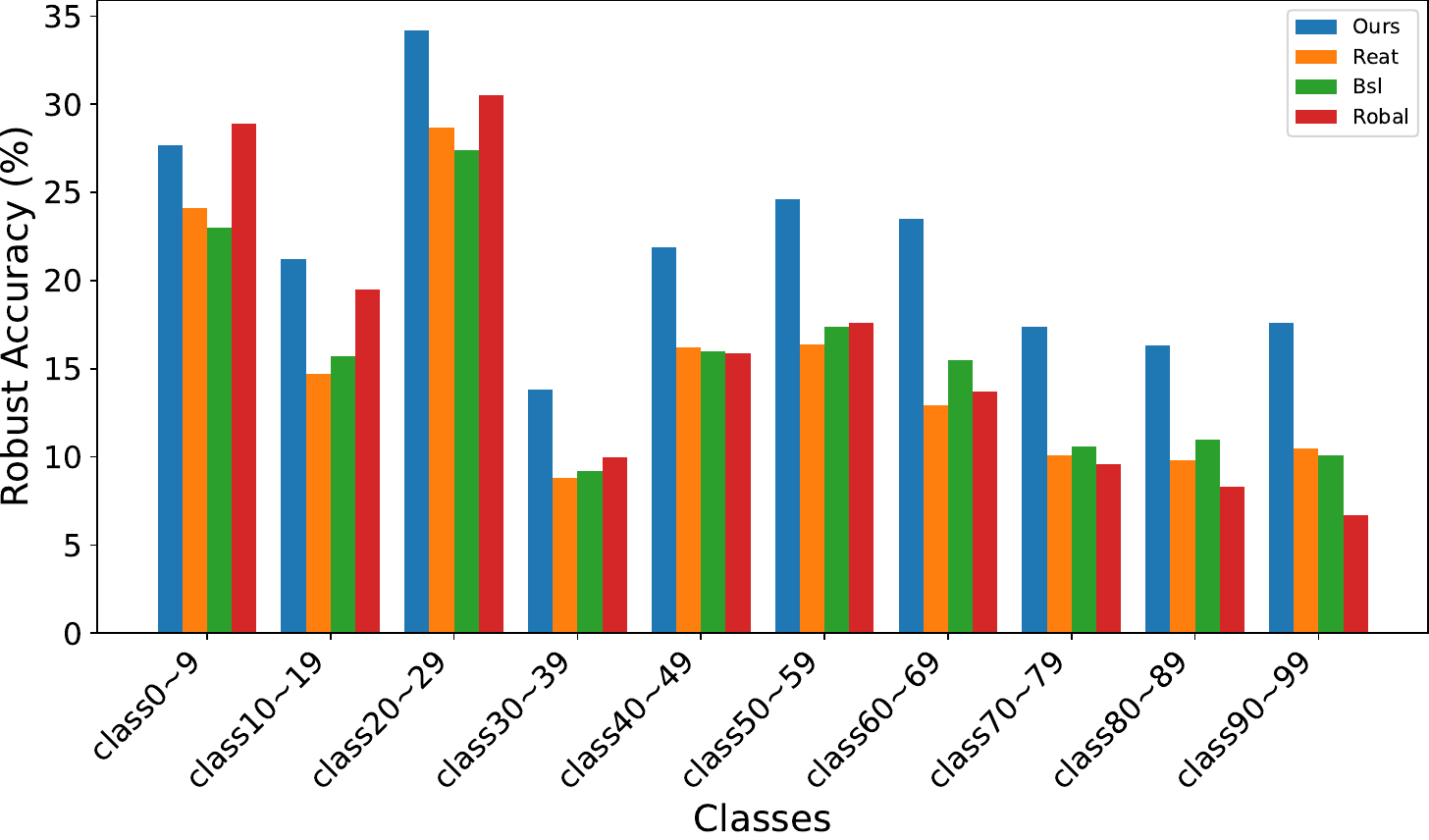}
 \caption{Class-wise robustness.}
   \label{fig:class_wise}
\end{figure}
In \Cref{fig:class_wise}, we divided the classes of CIFAR-100 into 10 groups and measured the robustness across them. As we move from class 0 to class 99, the number of data points decreases. Our method demonstrated higher robustness across all class groups compared to the baseline. Notably, it achieved the best performance in all groups except the first group. In contrast, Robal showed strong performance on the first group (head classes) but the worst performance on the last group (tail classes).

\end{document}

%% file: math_commands.tex
\usepackage{amsmath,amsfonts,bm}

\def\eqref#1{equation~\ref{#1}}

\def\1{\bm{1}}

\DeclareMathAlphabet{\mathsfit}{\encodingdefault}{\sfdefault}{m}{sl}
\SetMathAlphabet{\mathsfit}{bold}{\encodingdefault}{\sfdefault}{bx}{n}













%% file: main.bbl
\begin{thebibliography}{56}
\providecommand{\natexlab}[1]{#1}
\providecommand{\url}[1]{\texttt{#1}}
\expandafter\ifx\csname urlstyle\endcsname\relax
  \providecommand{\doi}[1]{doi: #1}\else
  \providecommand{\doi}{doi: \begingroup \urlstyle{rm}\Url}\fi

\bibitem[Alshammari et~al.(2022)Alshammari, Wang, Ramanan, and Kong]{alshammari2022long}
Shaden Alshammari, Yu-Xiong Wang, Deva Ramanan, and Shu Kong.
\newblock Long-tailed recognition via weight balancing.
\newblock In \emph{Proceedings of the IEEE/CVF Conference on Computer Vision and Pattern Recognition}, pp.\  6897--6907, 2022.

\bibitem[Athalye et~al.(2018)Athalye, Carlini, and Wagner]{athalye2018obfuscated}
Anish Athalye, Nicholas Carlini, and David Wagner.
\newblock Obfuscated gradients give a false sense of security: Circumventing defenses to adversarial examples.
\newblock In \emph{International conference on machine learning}, pp.\  274--283. PMLR, 2018.

\bibitem[Bai et~al.(2021{\natexlab{a}})Bai, Luo, Zhao, Wen, and Wang]{bai2021recent}
Tao Bai, Jinqi Luo, Jun Zhao, Bihan Wen, and Qian Wang.
\newblock Recent advances in adversarial training for adversarial robustness.
\newblock \emph{arXiv preprint arXiv:2102.01356}, 2021{\natexlab{a}}.

\bibitem[Bai et~al.(2021{\natexlab{b}})Bai, Zeng, Jiang, Xia, Ma, and Wang]{bai2021improving}
Yang Bai, Yuyuan Zeng, Yong Jiang, Shu-Tao Xia, Xingjun Ma, and Yisen Wang.
\newblock Improving adversarial robustness via channel-wise activation suppressing.
\newblock \emph{arXiv preprint arXiv:2103.08307}, 2021{\natexlab{b}}.

\bibitem[Cao et~al.(2019{\natexlab{a}})Cao, Wei, Gaidon, Arechiga, and Ma]{cao2019learning}
Kaidi Cao, Colin Wei, Adrien Gaidon, Nikos Arechiga, and Tengyu Ma.
\newblock Learning imbalanced datasets with label-distribution-aware margin loss.
\newblock \emph{Advances in neural information processing systems}, 32, 2019{\natexlab{a}}.

\bibitem[Cao et~al.(2019{\natexlab{b}})Cao, Wei, Gaidon, Arechiga, and Ma]{class_margin_lt}
Kaidi Cao, Colin Wei, Adrien Gaidon, Nikos Arechiga, and Tengyu Ma.
\newblock Learning imbalanced datasets with label-distribution-aware margin loss.
\newblock \emph{Advances in neural information processing systems}, 32, 2019{\natexlab{b}}.

\bibitem[Carlini \& Wagner(2017)Carlini and Wagner]{CW_attack}
Nicholas Carlini and David Wagner.
\newblock Towards evaluating the robustness of neural networks.
\newblock In \emph{2017 ieee symposium on security and privacy (sp)}, pp.\  39--57. Ieee, 2017.

\bibitem[Carmon et~al.(2019)Carmon, Raghunathan, Schmidt, Duchi, and Liang]{carmon2019unlabeled}
Yair Carmon, Aditi Raghunathan, Ludwig Schmidt, John~C Duchi, and Percy~S Liang.
\newblock Unlabeled data improves adversarial robustness.
\newblock \emph{Advances in neural information processing systems}, 32, 2019.

\bibitem[Chawla et~al.(2002)Chawla, Bowyer, Hall, and Kegelmeyer]{chawla2002smote}
Nitesh~V Chawla, Kevin~W Bowyer, Lawrence~O Hall, and W~Philip Kegelmeyer.
\newblock Smote: synthetic minority over-sampling technique.
\newblock \emph{Journal of artificial intelligence research}, 16:\penalty0 321--357, 2002.

\bibitem[Cohen et~al.(2019)Cohen, Rosenfeld, and Kolter]{cohen2019certified}
Jeremy Cohen, Elan Rosenfeld, and Zico Kolter.
\newblock Certified adversarial robustness via randomized smoothing.
\newblock In \emph{international conference on machine learning}, pp.\  1310--1320. PMLR, 2019.

\bibitem[Croce \& Hein(2020)Croce and Hein]{AutoAttack}
Francesco Croce and Matthias Hein.
\newblock Reliable evaluation of adversarial robustness with an ensemble of diverse parameter-free attacks.
\newblock In \emph{International conference on machine learning}, pp.\  2206--2216. PMLR, 2020.

\bibitem[Cubuk et~al.(2019)Cubuk, Zoph, Mane, Vasudevan, and Le]{cubuk2019autoaugment}
Ekin~D Cubuk, Barret Zoph, Dandelion Mane, Vijay Vasudevan, and Quoc~V Le.
\newblock Autoaugment: Learning augmentation strategies from data.
\newblock In \emph{Proceedings of the IEEE/CVF conference on computer vision and pattern recognition}, pp.\  113--123, 2019.

\bibitem[Cubuk et~al.(2020)Cubuk, Zoph, Shlens, and Le]{cubuk2020randaugment}
Ekin~D Cubuk, Barret Zoph, Jonathon Shlens, and Quoc~V Le.
\newblock Randaugment: Practical automated data augmentation with a reduced search space.
\newblock In \emph{Proceedings of the IEEE/CVF conference on computer vision and pattern recognition workshops}, pp.\  702--703, 2020.

\bibitem[Cui et~al.(2019)Cui, Jia, Lin, Song, and Belongie]{cui2019class}
Yin Cui, Menglin Jia, Tsung-Yi Lin, Yang Song, and Serge Belongie.
\newblock Class-balanced loss based on effective number of samples.
\newblock In \emph{Proceedings of the IEEE/CVF conference on computer vision and pattern recognition}, pp.\  9268--9277, 2019.

\bibitem[Das et~al.(2017)Das, Shanbhogue, Chen, Hohman, Chen, Kounavis, and Chau]{2017_jpeg_defense}
Nilaksh Das, Madhuri Shanbhogue, Shang-Tse Chen, Fred Hohman, Li~Chen, Michael~E Kounavis, and Duen~Horng Chau.
\newblock Keeping the bad guys out: Protecting and vaccinating deep learning with jpeg compression.
\newblock \emph{arXiv preprint arXiv:1705.02900}, 2017.

\bibitem[Du et~al.(2023)Du, Yang, Jia, Nan, Chen, and Yang]{du2023global}
Fei Du, Peng Yang, Qi~Jia, Fengtao Nan, Xiaoting Chen, and Yun Yang.
\newblock Global and local mixture consistency cumulative learning for long-tailed visual recognitions.
\newblock In \emph{Proceedings of the IEEE/CVF Conference on Computer Vision and Pattern Recognition}, pp.\  15814--15823, 2023.

\bibitem[Goldblum et~al.(2020)Goldblum, Fowl, Feizi, and Goldstein]{ard}
Micah Goldblum, Liam Fowl, Soheil Feizi, and Tom Goldstein.
\newblock Adversarially robust distillation.
\newblock In \emph{Proceedings of the AAAI Conference on Artificial Intelligence}, volume~34, pp.\  3996--4003, 2020.

\bibitem[Goodfellow et~al.(2014)Goodfellow, Shlens, and Szegedy]{FGSM}
Ian~J Goodfellow, Jonathon Shlens, and Christian Szegedy.
\newblock Explaining and harnessing adversarial examples.
\newblock \emph{arXiv preprint arXiv:1412.6572}, 2014.

\bibitem[Grigorescu et~al.(2020)Grigorescu, Trasnea, Cocias, and Macesanu]{safe2}
Sorin Grigorescu, Bogdan Trasnea, Tiberiu Cocias, and Gigel Macesanu.
\newblock A survey of deep learning techniques for autonomous driving.
\newblock \emph{Journal of Field Robotics}, 37\penalty0 (3):\penalty0 362--386, 2020.

\bibitem[Han et~al.(2005)Han, Wang, and Mao]{han2005borderline}
Hui Han, Wen-Yuan Wang, and Bing-Huan Mao.
\newblock Borderline-smote: a new over-sampling method in imbalanced data sets learning.
\newblock In \emph{International conference on intelligent computing}, pp.\  878--887. Springer, 2005.

\bibitem[He et~al.(2016{\natexlab{a}})He, Zhang, Ren, and Sun]{he2016deep}
Kaiming He, Xiangyu Zhang, Shaoqing Ren, and Jian Sun.
\newblock Deep residual learning for image recognition.
\newblock In \emph{Proceedings of the IEEE conference on computer vision and pattern recognition}, pp.\  770--778, 2016{\natexlab{a}}.

\bibitem[He et~al.(2016{\natexlab{b}})He, Zhang, Ren, and Sun]{he2016identity}
Kaiming He, Xiangyu Zhang, Shaoqing Ren, and Jian Sun.
\newblock Identity mappings in deep residual networks.
\newblock In \emph{Computer Vision--ECCV 2016: 14th European Conference, Amsterdam, The Netherlands, October 11--14, 2016, Proceedings, Part IV 14}, pp.\  630--645. Springer, 2016{\natexlab{b}}.

\bibitem[Huang et~al.(2023)Huang, Chen, Wang, Lu, Cheng, and Wang]{adaad}
Bo~Huang, Mingyang Chen, Yi~Wang, Junda Lu, Minhao Cheng, and Wei Wang.
\newblock Boosting accuracy and robustness of student models via adaptive adversarial distillation.
\newblock In \emph{Proceedings of the IEEE/CVF Conference on Computer Vision and Pattern Recognition}, pp.\  24668--24677, 2023.

\bibitem[Jin et~al.(2022)Jin, Yi, Huang, Schewe, and Huang]{jin2022enhancing}
Gaojie Jin, Xinping Yi, Wei Huang, Sven Schewe, and Xiaowei Huang.
\newblock Enhancing adversarial training with second-order statistics of weights.
\newblock In \emph{Proceedings of the IEEE/CVF Conference on Computer Vision and Pattern Recognition}, pp.\  15273--15283, 2022.

\bibitem[Jin et~al.(2023)Jin, Yi, Wu, Mu, and Huang]{jin2023randomized}
Gaojie Jin, Xinping Yi, Dengyu Wu, Ronghui Mu, and Xiaowei Huang.
\newblock Randomized adversarial training via taylor expansion.
\newblock In \emph{Proceedings of the IEEE/CVF Conference on Computer Vision and Pattern Recognition}, pp.\  16447--16457, 2023.

\bibitem[Kang et~al.(2019)Kang, Xie, Rohrbach, Yan, Gordo, Feng, and Kalantidis]{kang2019decoupling}
Bingyi Kang, Saining Xie, Marcus Rohrbach, Zhicheng Yan, Albert Gordo, Jiashi Feng, and Yannis Kalantidis.
\newblock Decoupling representation and classifier for long-tailed recognition.
\newblock \emph{arXiv preprint arXiv:1910.09217}, 2019.

\bibitem[Krizhevsky et~al.(2009)Krizhevsky, Hinton, et~al.]{krizhevsky2009learning}
Alex Krizhevsky, Geoffrey Hinton, et~al.
\newblock Learning multiple layers of features from tiny images.
\newblock 2009.

\bibitem[Le \& Yang(2015)Le and Yang]{le2015tiny}
Ya~Le and Xuan Yang.
\newblock Tiny imagenet visual recognition challenge.
\newblock \emph{CS 231N}, 7\penalty0 (7):\penalty0 3, 2015.

\bibitem[Lee et~al.(2023)Lee, Cho, and Kim]{IGDM}
Hongsin Lee, Seungju Cho, and Changick Kim.
\newblock Indirect gradient matching for adversarial robust distillation.
\newblock \emph{arXiv preprint arXiv:2312.03286}, 2023.

\bibitem[Lee et~al.(2024)Lee, Lee, Jang, Park, Bae, and Yoon]{DAFA}
Hyungyu Lee, Saehyung Lee, Hyemi Jang, Junsung Park, Ho~Bae, and Sungroh Yoon.
\newblock Dafa: Distance-aware fair adversarial training.
\newblock \emph{arXiv preprint arXiv:2401.12532}, 2024.

\bibitem[Li et~al.(2023)Li, Xu, and Zhang]{li2023alleviating}
Guanlin Li, Guowen Xu, and Tianwei Zhang.
\newblock Alleviating the effect of data imbalance on adversarial training, 2023.

\bibitem[Li et~al.(2021)Li, Wang, and Wu]{li2021self}
Tianhao Li, Limin Wang, and Gangshan Wu.
\newblock Self supervision to distillation for long-tailed visual recognition.
\newblock In \emph{Proceedings of the IEEE/CVF international conference on computer vision}, pp.\  630--639, 2021.

\bibitem[Ma et~al.(2021)Ma, Niu, Gu, Wang, Zhao, Bailey, and Lu]{safe1}
Xingjun Ma, Yuhao Niu, Lin Gu, Yisen Wang, Yitian Zhao, James Bailey, and Feng Lu.
\newblock Understanding adversarial attacks on deep learning based medical image analysis systems.
\newblock \emph{Pattern Recognition}, 110:\penalty0 107332, 2021.

\bibitem[Madry et~al.(2017)Madry, Makelov, Schmidt, Tsipras, and Vladu]{PGD}
Aleksander Madry, Aleksandar Makelov, Ludwig Schmidt, Dimitris Tsipras, and Adrian Vladu.
\newblock Towards deep learning models resistant to adversarial attacks.
\newblock \emph{arXiv preprint arXiv:1706.06083}, 2017.

\bibitem[Maroto et~al.(2022)Maroto, Ortiz{-}Jim{\'{e}}nez, and Frossard]{akd}
Javier Maroto, Guillermo Ortiz{-}Jim{\'{e}}nez, and Pascal Frossard.
\newblock On the benefits of knowledge distillation for adversarial robustness.
\newblock \emph{CoRR}, abs/2203.07159, 2022.
\newblock \doi{10.48550/ARXIV.2203.07159}.
\newblock URL \url{https://doi.org/10.48550/arXiv.2203.07159}.

\bibitem[Menon et~al.(2020)Menon, Jayasumana, Rawat, Jain, Veit, and Kumar]{menon2020long}
Aditya~Krishna Menon, Sadeep Jayasumana, Ankit~Singh Rawat, Himanshu Jain, Andreas Veit, and Sanjiv Kumar.
\newblock Long-tail learning via logit adjustment.
\newblock \emph{arXiv preprint arXiv:2007.07314}, 2020.

\bibitem[Pang et~al.(2020)Pang, Yang, Dong, Su, and Zhu]{pang2020bag}
Tianyu Pang, Xiao Yang, Yinpeng Dong, Hang Su, and Jun Zhu.
\newblock Bag of tricks for adversarial training.
\newblock \emph{arXiv preprint arXiv:2010.00467}, 2020.

\bibitem[Qin et~al.(2019)Qin, Martens, Gowal, Krishnan, Dvijotham, Fawzi, De, Stanforth, and Kohli]{qin2019adversarial}
Chongli Qin, James Martens, Sven Gowal, Dilip Krishnan, Krishnamurthy Dvijotham, Alhussein Fawzi, Soham De, Robert Stanforth, and Pushmeet Kohli.
\newblock Adversarial robustness through local linearization.
\newblock \emph{Advances in Neural Information Processing Systems}, 32, 2019.

\bibitem[Ren et~al.(2020)Ren, Yu, Ma, Zhao, Yi, et~al.]{ren2020balanced}
Jiawei Ren, Cunjun Yu, Xiao Ma, Haiyu Zhao, Shuai Yi, et~al.
\newblock Balanced meta-softmax for long-tailed visual recognition.
\newblock \emph{Advances in neural information processing systems}, 33:\penalty0 4175--4186, 2020.

\bibitem[Tack et~al.(2022)Tack, Yu, Jeong, Kim, Hwang, and Shin]{tack2022consistency}
Jihoon Tack, Sihyun Yu, Jongheon Jeong, Minseon Kim, Sung~Ju Hwang, and Jinwoo Shin.
\newblock Consistency regularization for adversarial robustness.
\newblock In \emph{Proceedings of the AAAI Conference on Artificial Intelligence}, volume~36, pp.\  8414--8422, 2022.

\bibitem[Tan et~al.(2020)Tan, Wang, Li, Li, Ouyang, Yin, and Yan]{tan2020equalization}
Jingru Tan, Changbao Wang, Buyu Li, Quanquan Li, Wanli Ouyang, Changqing Yin, and Junjie Yan.
\newblock Equalization loss for long-tailed object recognition.
\newblock In \emph{Proceedings of the IEEE/CVF conference on computer vision and pattern recognition}, pp.\  11662--11671, 2020.

\bibitem[Wang et~al.(2023)Wang, Luo, Sato, Xu, and Chen]{wang2023does}
Ningfei Wang, Yunpeng Luo, Takami Sato, Kaidi Xu, and Qi~Alfred Chen.
\newblock Does physical adversarial example really matter to autonomous driving? towards system-level effect of adversarial object evasion attack.
\newblock In \emph{Proceedings of the IEEE/CVF International Conference on Computer Vision}, pp.\  4412--4423, 2023.

\bibitem[Wang et~al.(2020)Wang, Zou, Yi, Bailey, Ma, and Gu]{MART}
Yisen Wang, Difan Zou, Jinfeng Yi, James Bailey, Xingjun Ma, and Quanquan Gu.
\newblock Improving adversarial robustness requires revisiting misclassified examples.
\newblock In \emph{International Conference on Learning Representations}, 2020.

\bibitem[Wei et~al.(2023)Wei, Wang, Guo, and Wang]{wei2023cfa}
Zeming Wei, Yifei Wang, Yiwen Guo, and Yisen Wang.
\newblock Cfa: Class-wise calibrated fair adversarial training.
\newblock In \emph{Proceedings of the IEEE/CVF Conference on Computer Vision and Pattern Recognition}, pp.\  8193--8201, 2023.

\bibitem[Wu et~al.(2020)Wu, Xia, and Wang]{2020_awp}
Dongxian Wu, Shu-Tao Xia, and Yisen Wang.
\newblock Adversarial weight perturbation helps robust generalization.
\newblock \emph{Advances in Neural Information Processing Systems}, 33:\penalty0 2958--2969, 2020.

\bibitem[Wu et~al.(2021)Wu, Liu, Huang, Wang, and Lin]{Wu_2021_CVPR}
Tong Wu, Ziwei Liu, Qingqiu Huang, Yu~Wang, and Dahua Lin.
\newblock Adversarial robustness under long-tailed distribution.
\newblock In \emph{Proceedings of the IEEE/CVF Conference on Computer Vision and Pattern Recognition (CVPR)}, pp.\  8659--8668, June 2021.

\bibitem[Xie et~al.(2019)Xie, Wu, Maaten, Yuille, and He]{xie2019feature}
Cihang Xie, Yuxin Wu, Laurens van~der Maaten, Alan~L Yuille, and Kaiming He.
\newblock Feature denoising for improving adversarial robustness.
\newblock In \emph{Proceedings of the IEEE/CVF conference on computer vision and pattern recognition}, pp.\  501--509, 2019.

\bibitem[Xu et~al.(2021)Xu, Liu, Li, Jain, and Tang]{FairAT}
Han Xu, Xiaorui Liu, Yaxin Li, Anil Jain, and Jiliang Tang.
\newblock To be robust or to be fair: Towards fairness in adversarial training.
\newblock In \emph{International conference on machine learning}, pp.\  11492--11501. PMLR, 2021.

\bibitem[Yue et~al.(2024)Yue, Mou, Wang, and Zhao]{Yue_2024_CVPR}
Xinli Yue, Ningping Mou, Qian Wang, and Lingchen Zhao.
\newblock Revisiting adversarial training under long-tailed distributions.
\newblock In \emph{Proceedings of the IEEE/CVF Conference on Computer Vision and Pattern Recognition (CVPR)}, pp.\  24492--24501, June 2024.

\bibitem[Zagoruyko \& Komodakis(2016)Zagoruyko and Komodakis]{zagoruyko2016wide}
Sergey Zagoruyko and Nikos Komodakis.
\newblock Wide residual networks.
\newblock \emph{arXiv preprint arXiv:1605.07146}, 2016.

\bibitem[Zhang et~al.(2019)Zhang, Yu, Jiao, Xing, El~Ghaoui, and Jordan]{TRADES}
Hongyang Zhang, Yaodong Yu, Jiantao Jiao, Eric Xing, Laurent El~Ghaoui, and Michael Jordan.
\newblock Theoretically principled trade-off between robustness and accuracy.
\newblock In \emph{International conference on machine learning}, pp.\  7472--7482. PMLR, 2019.

\bibitem[Zhang et~al.(2022)Zhang, Hu, Sun, Chen, and Mao]{zhang2022adversarial}
Qingzhao Zhang, Shengtuo Hu, Jiachen Sun, Qi~Alfred Chen, and Z~Morley Mao.
\newblock On adversarial robustness of trajectory prediction for autonomous vehicles.
\newblock In \emph{Proceedings of the IEEE/CVF Conference on Computer Vision and Pattern Recognition}, pp.\  15159--15168, 2022.

\bibitem[Zhang et~al.(2021)Zhang, Li, Yan, He, and Sun]{zhang2021distribution}
Songyang Zhang, Zeming Li, Shipeng Yan, Xuming He, and Jian Sun.
\newblock Distribution alignment: A unified framework for long-tail visual recognition.
\newblock In \emph{Proceedings of the IEEE/CVF conference on computer vision and pattern recognition}, pp.\  2361--2370, 2021.

\bibitem[Zhou et~al.(2020)Zhou, Cui, Wei, and Chen]{zhou2020bbn}
Boyan Zhou, Quan Cui, Xiu-Shen Wei, and Zhao-Min Chen.
\newblock Bbn: Bilateral-branch network with cumulative learning for long-tailed visual recognition.
\newblock In \emph{Proceedings of the IEEE/CVF conference on computer vision and pattern recognition}, pp.\  9719--9728, 2020.

\bibitem[Zhu et~al.(2021)Zhu, Yao, Han, Zhang, Liu, Niu, Zhou, Xu, and Yang]{iad}
Jianing Zhu, Jiangchao Yao, Bo~Han, Jingfeng Zhang, Tongliang Liu, Gang Niu, Jingren Zhou, Jianliang Xu, and Hongxia Yang.
\newblock Reliable adversarial distillation with unreliable teachers.
\newblock \emph{arXiv preprint arXiv:2106.04928}, 2021.

\bibitem[Zi et~al.(2021)Zi, Zhao, Ma, and Jiang]{rslad}
Bojia Zi, Shihao Zhao, Xingjun Ma, and Yu-Gang Jiang.
\newblock Revisiting adversarial robustness distillation: Robust soft labels make student better.
\newblock In \emph{Proceedings of the IEEE/CVF International Conference on Computer Vision}, pp.\  16443--16452, 2021.

\end{thebibliography}
